\documentclass[preprint,12pt]{elsarticle}

\usepackage{amsmath,amsfonts,amsthm,amssymb,paralist,subfigure,graphicx,amsbsy,float,epsfig,color}
\usepackage{cuted,mathtools,lipsum}

\usepackage[ruled,vlined,linesnumbered]{algorithm2e}
\usepackage{stmaryrd,url}

\usepackage{amsthm}
\newtheorem{lem}{Lemma}
\newtheorem{ass}{Assumption}
\newtheorem{theorem}{Theorem}

\newtheorem{rem}{Remark}

\def\mb{\mathbf}
\def\mbb{\mathbb}
\def\mc{\mathcal}

\def\mb{\mathbf}
\def\mbb{\mathbb}
\def\mc{\mathcal}

\DeclareMathOperator*{\argmin}{argmin}

\journal{Engineering Applications of Artificial Intelligence}

\begin{document}

\begin{frontmatter}

\title{Machine Learning and CPU (Central Processing Unit) Scheduling Co-Optimization over a Network of Computing Centers
}

\author[Sem]{Mohammadreza Doostmohammadian}
\affiliation[Sem]{Mechatronics Group, Faculty of Mechanical Engineering, Semnan University, Semnan, Iran, and Center for International Scientific Studies and Collaborations, Tehran, Iran, doost@semnan.ac.ir.}

\author[kaz]{ Zulfiya R. Gabidullina}
\affiliation[kaz]{Institute of Computational Mathematics and Information Technologies, Kazan Federal University, Russia, Zulfiya.Gabidullina@kpfu.ru.}

\author[HR]{ Hamid R. Rabiee}
\affiliation[HR]{Computer Engineering Department, Sharif University of Technology, Tehran, Iran,
	rabiee@sharif.edu.}

\begin{abstract}
	In the rapidly evolving research on artificial intelligence (AI) the demand for fast, computationally efficient, and scalable solutions has increased in recent years. The problem of optimizing the computing resources for distributed machine learning (ML) and optimization is considered in this paper. Given a set of data distributed over a network of computing-nodes/servers, the idea is to optimally assign the CPU (central processing unit) usage while simultaneously training each computing node locally via its own share of data. This formulates the problem as a co-optimization setup to (i) optimize the data processing and (ii) optimally allocate the computing resources. The information-sharing network among the nodes might be time-varying, but with balanced weights to ensure consensus-type convergence of the algorithm. The algorithm is all-time feasible, which implies that the computing resource-demand balance constraint holds at all iterations of the proposed solution. Moreover, the solution allows addressing possible log-scale quantization over the information-sharing channels to exchange log-quantized data. For some example applications, distributed support-vector-machine (SVM) and regression are considered as the ML training models.
	Results from perturbation theory, along with Lyapunov stability and eigen-spectrum analysis, are used to prove the convergence towards the optimal case. As compared to existing CPU scheduling solutions, the proposed algorithm improves the cost optimality gap by more than $50\%$. 
\end{abstract}

%%Graphical abstract
\begin{graphicalabstract}
	\includegraphics{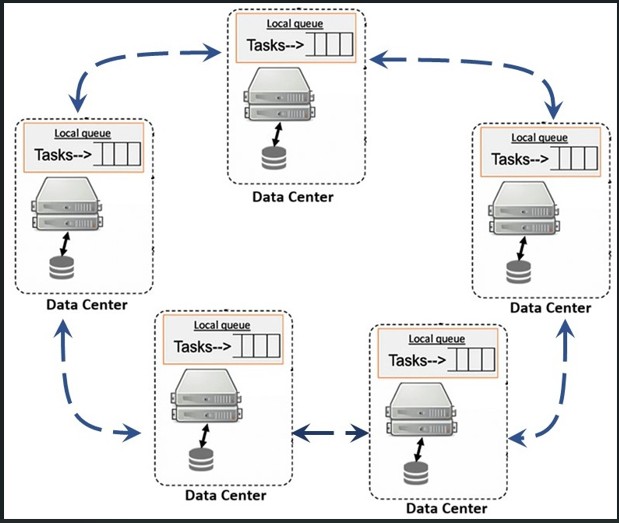}
\end{graphicalabstract}

%%Research highlights
\begin{highlights}
	\item Simultaneous optimization of data processing and computing resource allocation over a network of computing nodes.
	\item Preserving resource-demand constraint feasibility at all iterations
	\item Addressing log-scale quantization for limited networking traffic setups.
	\item Perturbation-based convergence analysis for time-varying data-sharing networks.
	\item Applicability in distributed support-vector machines and regression models.
\end{highlights}

\begin{keyword}
	Distributed optimization \sep networked data centres\sep consensus\sep computing resource scheduling
\end{keyword}

\end{frontmatter}

\section{Introduction} \label{sec_intro}
Data processing and distributed machine-learning (ML) over a network of computing centers has been recently motivated by advances in cloud/edge-computing \cite{ARANDA2022107826}, multi-agent learning systems \cite{zamani2016iterative,WANG2025109825}, Internet-of-Things (IoT) applications \cite{yosuf2020energy}, and cyber-physical-systems (CPS) \cite{CAPOTA2019204}. In this direction, \textit{distributed} and \textit{decentralized} algorithms are proposed to distribute the computational/processing loads among a group of computing nodes to advance the traditional centralized solutions in terms of scalability, robustness to single-node-of-failure, and reliability. Such algorithms allow for combining machine learning and data filtering models across a network of interconnected nodes/agents and cloud-based platforms \cite{MELONI2018156}, enabling parallel processing and accelerated decision-making \cite{di2020distributed,spl24}. In these setups, one main problem is \textit{how to efficiently/optimally schedule the computational resources among these computing nodes while performing optimization-based machine learning and data filtering tasks}. The synergy between CPU scheduling and learning from data offers promising solutions for optimizing computing performance in distributed ML setups.

Some related application-based works in the literature are discussed next. A relevant problem of interest is mixed-integer linear programming with applications to scheduling \cite{liu2021novel} and scalable task assignment \cite{bragin2018scalable}, which is further extended to asynchronous coordination via primal-dual formulation in \cite{bragin2020distributed}. Distributed task allocation to reach consensus on duty to capability ratio \cite{MiadCons} and facility location optimization for discrete coverage algorithms over a convex polygon \cite{MSC09} via swarm robotic networks are also based on distributed scheduling and optimization algorithms.  On the other hand, coordination of CPUs over networked data centres using both centralized algorithms \cite{grammenos2023cpu,kalyvianaki2009self} and quantized consensus-based distributed algorithms \cite{rikos2021optimal} is considered in the literature.
Distributed resource allocation and optimization algorithms with applications in energy system management are also discussed in the literature; e.g., see relevant works on convex quadratic relaxations of the AC power flow \cite{koutsoukis2016online}, distributed economic dispatch to coordinate generators and iteratively share local information to determine their outputs that minimizes total production cost while meeting demand and operational constraints \cite{ZHOLBARYSSOV201947}, distributed generator scheduling to cooperatively manage energy production across a grid network to optimize reliability, cost, and grid services while respecting local constraints \cite{scl2023,cheng2024distributed}, and distributed automatic generation control for power regulation tasks to achieve system-wide load-generation balance without centralized control \cite{scl,dsp}. Another concern in the literature is non-ideal data sharing due to quantization.
In this direction, a distributed subgradient method with adaptive (or dynamic) quantization is proposed in \cite{doan2020fast} while traditional static quantization with a fixed quantization level is considered for consensus \cite{rikos2022non}, learning \cite{wu2018error}, and optimization \cite{kajiyama2020linear}. One main challenge in the uniformly quantized setup is how to mitigate the quantization error. In this direction, one solution is to use \textit{log-scale} quantization instead. Another concern is the notion of constraint feasibility in distributed optimization. All-time feasibility is crucial to ensure that the (computing) resource-demand balance always holds. This is in contrast to primal-dual-based alternating direction method of multipliers (ADMM) formulations that reach resource-demand feasibility asymptotically \cite{jian2019distributed,falsone2023augmented,gong2024primal,cheng2024distributed}. What is missing in the existing literature is a unified framework to simultaneously address all-time feasible scheduling of the computing resources along with distributed learning and optimization while addressing log-scale quantization.

The main contributions of this work are as follows. Given a co-optimization problem for CPU scheduling and ML optimization, we provide a distributed algorithm to optimally allocate the computing resources while simultaneously solving the distributed optimization. This coordination algorithm is based on consensus mechanism and gradient descent tracked by introducing an auxiliary variable. We further address the logarithmic quantization of the data shared over the data-sharing network of computing-nodes/agents. Using perturbation theory and Lyapunov-based eigenspectrum analysis, we prove convergence to the optimal case in the presence of log-scale quantized information exchange. Further, the algorithm allows to consider change in the topology of the networked agents without violating the convergence of the solution. Moreover, we show that the algorithm is all-time feasible, implying that the (computing) resource-demand holds at all times along the solution. This implies that at any termination time of the algorithm, there is no violation in the balance between  CPU resources and computing demand. We verify the feasibility and optimal convergence by extensive simulations for different ML applications. Comparisons with the existing literature are also provided by the simulations.

\textit{Paper Organization:} Section~\ref{sec_prob} formulates the problem in mathematical form. Section~\ref{sec_alg} provides the main algorithm to solve the problem. Section~\ref{sec_analysis} provides the analysis of convergence, optimality, and feasibility. Section~\ref{sec_sim} presents the simulation results. Finally, Section~\ref{sec_con} concludes the paper.

\textit{Notations:} $I_m$ is the identity matrix of size $m$. $\mb{1}_n$ (or $\mb{0}_n$)  is the column vector of all ones (or zeros) of size $n$. Similarly, $\mb{0}_{n \times n}$ denotes zero square matrix of size $n$. Operator $\otimes$ denotes the Kronecker product. Operator "$;$" in vectors implies column concatenation.

\section{Problem Formulation}  \label{sec_prob}
The optimization problem in this paper is two-fold: (i) optimizing the allocation of CPU resources and (ii) optimizing the loss function associated with the machine learning application.
\begin{align} \label{eq_opt}
\min_{\mb{x},\mb{y}}
~ & \sum_{i=1}^{n} f_i(\mb{x}_i) + g_i(\mb{y}_i)\\ \nonumber
& \text{s.t.} ~~ \mb{y}_1 = \mb{y}_2 = \dots = \mb{y}_n \\\nonumber
&  \sum_{i=1}^{n} \mb{x}_i = b \\\nonumber
& \mb{x}_i \in \Xi_i
\end{align}
with $\mb{y}_i \in \mathbb{R}^m$ as the ML parameter states, $g_i: \mathbb{R}^m \rightarrow \mathbb{R}$ as the ML loss function, $\mb{x}_i \in \mathbb{R}$ as the assigned CPU resources, $f_i: \mathbb{R} \rightarrow \mathbb{R}$ as the cost of allocated CPU to node $i$, and $\Xi_i \subseteq \mathbb{R}$ representing a range of admissible values for CPU states, all at computing node $i$. In general, $\Xi_i = \{\mb{x} \in \mathbb{R}: h_i^j(\mb{x}) \leq 0, j=1,\dots,p_i \}$
with $h_i^j:\mathbb{R} \rightarrow \mathbb{R}$ as convex and smooth functions on $\Xi_i$.  One simple example of the latter is the so-called box constraints $\mb{x}_i \in [m_i~ M_i]$. Parameter $b$ denotes the overall available CPU resources to all nodes. The global state vectors are defined as $\mb{x} := [\mb{x}_1;\dots;\mb{x}_n]$ and $\mb{y} := [\mb{y}_1;\dots;\mb{y}_n]$.

\begin{ass} \label{ass_cost}
	Local CPU cost function $f_i(\cdot)$ is strictly convex and smooth. The ML loss function $g_i(\cdot)$ is (possibly) non-convex and smooth, satisfying:
	\begin{align} \label{eq_H}
		(\mb{1}_n \otimes I_m)^\top H (\mb{1}_n \otimes I_m) \succ 0.
	\end{align}
	with $H:=\mbox{diag}[\nabla^2 g_i(\mb{y}_i)] \preceq L I_{mn}$.  
\end{ass}

\begin{rem} \label{rem_objective}
The CPU cost is typically modelled as a quadratic function as described in \cite{grammenos2023cpu,kalyvianaki2009self,rikos2021optimal}, which satisfies Assumption~\ref{ass_cost}. 
The examples for ML loss functions include Hinge loss for data classification via support-vector-machine (SVM) \cite{ddsvm}, linear and logistic regression costs \cite{xin2020decentralized,sundhar2012new}, and collaborative least square loss function for inference and estimation \cite{dimakis2010gossip,kar2008distributed,zhang2023top}. For most existing smooth cost/loss functions, Assumption~\ref{ass_cost} holds.
\end{rem}

The way data is distributed among the computing nodes relates the two objective functions $f_i(\cdot)$ and $g_i(\cdot)$. For example, the quantity/quality of the data accessible to each computing node defines both the share of assigned CPU resources $\mb{x}_i$ and the local ML parameters $\mb{y}_i$ associated with the data. The two optimization algorithms are genuinely coupled via the constraints, specifically the constraint $\mb{x}_i \in \Xi_i$, where the local constraint set $\Xi_i$ depends on the number of the data points (for ML objective) associated with the computing node $i$ (for scheduling objective) and ties the two problems. These constraints are typically addressed by the so-called box constraints in the literature \cite{nesterov1998introductory,bertsekas1975necessary}. From this point onward in the paper, we consider a specific objective function and provide the solution. However, the solution holds for the general model~\eqref{eq_opt} and can be easily extended to similar objective functions satisfying Assumption~\ref{ass_cost}.

\subsection{CPU Scheduling Model}
The CPU scheduling model in this subsection follows from \cite{grammenos2023cpu,kalyvianaki2009self,rikos2021optimal}.
Assume that the networked data centres consist of a set of (computing) nodes $\mathcal{V}$ of size $n$, where each node $i \in \mathcal{V}$ can also function as a resource scheduler, a common practice in modern data centres. Denote by $\mathcal{J}$ the set of all jobs to be scheduled. Each job $b_{j} \in \mathcal{J}$ requires $b_{j}$ CPU cycles to be performed, a quantity known prior to optimization.
At each computing node $i$, the total workload due to arriving jobs is $l_i$. The optimization period, $T_{h}$, represents the time during which the current set of jobs is performed before the next reallocation. Each node’s CPU capacity during this period is $\kappa_i^{\max} := c_i T_h$, with $c_i$ as the aggregate clock rate of all processing cores at node $i$, measured in cycles per second. The CPU availability at optimization step $k$ is $\kappa_i^{\mathrm{avail}}[k] = \kappa_i^{\max} - u_{i}[k]$, where $u_{i}[k]$ represents the cycles already allocated to running tasks.
Let $b[k] = \sum_{b_j[k] \in \mathcal{J}[k]} b_{j}[k]$ at step $k$, and $\kappa^{\mathrm{avail}}[k] = \sum_{i\in \mathcal{V}} \kappa_i^{\mathrm{avail}}[k]$ denote the total available capacity. The time horizon $T_{h}$ at step $k$ is chosen such that $b[k] \leq \kappa^{\mathrm{avail}}[k]$, ensuring that the total demand does not exceed the available resources. Each node calculates the optimal solution at every optimization step $k$ by executing a distributed algorithm that updates CPU utilization, accounting for (possibly log-quantized) information exchange and workloads. The algorithm ensures that each node balances its CPU utilization during task execution while meeting feasibility constraints. This \textit{balancing strategy} requires each node $i$ to calculate the optimal workload $\mb{x}_i^*[k]$ such that $\forall i,j \in \mathcal{V}$:
\begin{align} \label{eq_cond}
	\frac{\mb{x}_i^*[k] + u_{i}[k]}{\kappa_i^{\max}} = \frac{\mb{x}_j^*[k] + u_{j}[k]}{\kappa_j^{\max}} = \frac{b[k] + u_{\mathrm{tot}}[k]}{\kappa^{\max}},
\end{align}
where $\kappa^{\max} = \sum_{i\in \mathcal{V}} \kappa_i^{\max}$ and $u_{\mathrm{tot}}[k] = \sum_{i\in \mathcal{V}} u_{i}[k]$. This condition \eqref{eq_cond} implies that the task allocation strategy allows every node to balance its CPU utilization during the execution of the tasks, i.e., to coordinate the given tasks such that each node utilizes the same percentage of its own capacity while meeting the constraint feasibility.
For simplicity, we drop the index $k$ from this point onward in this section. Each node is associated with a scalar quadratic local cost function $f_i : \mathbb{R} \to \mathbb{R}$:
\begin{equation}\label{eq_local_cost_functions}
	f_i(\varkappa) = \frac{1}{2} \alpha_i (\varkappa - b_i)^2, 
\end{equation}
where $\alpha_i > 0$, $b_i \in \mathbb{R}$ is the positive demand at node $i$, and $\mb{\varkappa}=[\varkappa_1;\dots;\varkappa_{n}]$ is the global optimization parameter. The goal is to minimize the global cost function (sum of local cost functions).
Let the optimizer be
\begin{align}\label{eq_opt:1}
	\mb{\varkappa}^* = \arg\min_{\mb{\varkappa} \in \mathcal{Z}} \sum_{i \in \mathcal{V}} f_i(\varkappa_i),
\end{align}
where $\mathcal{Z}$ is the set of feasible values for $\mb{\varkappa}$, e.g., $\mathcal{Z}$ may represent the box constraints in the form $\underline{m}_i \leq \varkappa_i \leq \overline{M}_i$.
The closed-form solution for \eqref{eq_opt:1} is:
\begin{align}\label{eq_x}
	\mb{\varkappa}^* = \frac{\sum_{i \in \mathcal{V}} \alpha_i b_{i}}{\sum_{i \in \mathcal{V}} \alpha_i}.
\end{align}
From \cite{rikos2021optimal}, to find the optimal workload according to \eqref{eq_cond}, we need the solution of \eqref{eq_opt:1} to be 
\begin{align}\label{eq:closedform1}
	\mb{\varkappa}^* =  \frac{\sum_{v_{i} \in \mathcal{V}} \kappa_i^{\max} \frac{b_{i}+u_{i}}{\kappa_i^{\max}}}{\sum_{i \in \mathcal{V}} \kappa_i^{\max}} = \frac{\rho + u_{\mathrm{tot}}}{\kappa^{\max}}.
\end{align}
From \eqref{eq:closedform1}, we need to modify the local cost model \eqref{eq_local_cost_functions} as
\begin{align}\label{eq:fiz}
	f_i(\varkappa) = \frac{1}{2}\kappa_i^{\max} \left(\varkappa- \frac{b_{i}+u_{i}}{\kappa_i^{\max}} \right)^2.
\end{align}
This implies that every node $i$ finds its proportion of workload, and from this proportion valueit can find the optimal workload $\mb{x}_i^*$ to receive, which is 
\begin{align}\label{eq:optimal_workload}
	\mb{x}_i^*  = \frac{b + u_{\mathrm{tot}}}{\kappa^{\max}} \kappa_i^{\max} - u_{i}.
\end{align}
However, it should be noted that the allocated workload by \eqref{eq:optimal_workload} gives the optimal allocation subject to constraint \eqref{eq_cond}. From \cite{rikos2021optimal}, a more general improved cost model in the following form can be considered 
\begin{equation}\label{local_cost_functions2}
	f_i(\varkappa) = \dfrac{1}{2} \alpha_i (\varkappa_i - b_i)^2, 
\end{equation} 
where $\varkappa_i \neq \varkappa_j$, in general. Note the subtle difference here as the factors $\varkappa_i$ in \eqref{local_cost_functions2} could be unequal (compared to the same $\varkappa$ in formulation \eqref{eq_local_cost_functions}). Substituting $\mb{x}_i$ from \eqref{eq_cond}, 
\begin{equation}\label{local_cost_w}
	f_i(\mb{x}_i) = \dfrac{1}{2\kappa_i^{\max}}  (\mb{x}_i - b_i)^2,
\end{equation} 
This convex formulation gives lower cost by replacing the balancing constraint $\frac{\mb{x}_i + \rho_i}{\kappa_i^{\max}} = \frac{\mb{x}_j + b_j}{\kappa_j^{\max}}$ (or $\varkappa_i = \varkappa_j$) with a more general sum-preserving constraint $\sum_{i=1}^n \mb{x}_i = \sum_{i=1}^n \mb{x}_i^* = b$. This implies assigning the same workload as given by \eqref{eq:closedform1}. Then, the modified version of \eqref{eq:fiz} is 
\begin{align}\label{eq:fiz2}
	f_i(\mb{x}_i) = \frac{1}{2\kappa_i^{\max}} (\mb{x}_i - b_i)^2 ~\text{s.t.} ~ \sum_{i=1}^n \mb{x}_i  = b.
\end{align}
To avoid exceeding server capacities and increase mean response times, we add box constraints on load-to-capacity ratios, keeping them below $70\%-80\%$ of their capacity. This defines the local constraint $\Xi_i$.
The scheduling approach ensures efficient CPU resource allocation while maintaining balanced and feasible workloads across all servers, as discussed later in Section~\ref{sec_analysis}.

\begin{rem}
	The existing works for CPU allocation typically consider \textit{quadratic} cost models, see \cite{grammenos2023cpu,rikos2021optimal} for example. However, following Assumption~\ref{ass_cost}, this work considers \textit{strictly convex} cost models for CPU scheduling that could be non-quadratic in general. This is a more relaxed condition and allows for adding non-quadratic penalty terms (or barrier functions) to the objective function to address convex constraints or box constraints (this is discussed more in Section~\ref{sec_alg_lin}).
\end{rem}

\begin{rem} \label{rem_cpu}
The CPU balancing model in \cite{rikos2021optimal} assumes that all CPU states reach agreement on the assigned computing loads. This balancing assumption allows to apply consensus-based algorithms directly to reach the optimal value. On the other hand, in this work, we assume sum-preserving constraint $\sum_{i=1}^n \mb{x}_i = b$ instead. This is a less restrictive assumption on the allocated computing resources. In general, it is easy to show that the
allocation cost subject to the balancing constraint in \cite{rikos2021optimal} is always more than (or
equal to) the allocation cost subject to sum-preserving constraint. This is because the solution by \cite{rikos2021optimal} assigns the resources as $\mb{x}_i=\mb{x}_j=\tfrac{b}{n}$ to reach consensus for all $i,j$, which is more restrictive than $\sum_{i=1}^n \mb{x}_i = b$. This is
better illustrated in Section~\ref{sec_sim}.
\end{rem}

\subsection{ML Objective}
As stated in Remark~\eqref{rem_objective}, there are different objective functions to be optimized depending on the specific ML application. In this subsection, we consider two ML models in this paper. However, our solution holds for different ML objective functions satisfying Assumption~\ref{ass_cost}.

\subsubsection{Linear Regression}
Linear regression is a statistical ML approach that involves fitting a hyperplane to a set of data points to model the relationship between variables. Given a set of $N$ data points $\boldsymbol{\chi}_i \in \mathbb{R}^{m-1}$, which are distributed among nodes/agents $i=\{1,\hdots,n\}$, this model predicts the variables associated with the hyperplane $ \boldsymbol{\omega}^\top \boldsymbol{\chi}_i - \nu = a_i$ that is the best fit to the data. The approach to solving this problem is both centralized and distributed. In the centralized case, all the data points are gathered at the fusion centre to compute parameters $[\nu;\boldsymbol{\omega}]$ that solve the following,
\begin{equation} \label{eq_lr_cent}
\begin{aligned}
	\displaystyle
	& \min_{[\boldsymbol{\omega}^\top;\nu]}
	~ &  \sum_{i=1}^{N} (\boldsymbol{\omega}^\top \boldsymbol{\chi}_i - \nu -a_i)^2.
\end{aligned}
\end{equation}
This problem is also referred to as the linear least-squares problem. On the other hand, the decentralized case, instead of having all the data at one computing centre, distributes the data and computational load over a network of $n$ nodes.
Each computing mode $i$ takes its own share of data, denoted by $\boldsymbol{\chi}_i$, which satisfies $\frac{N}{n}\leq N_i\leq N$. Some data might be given to two or more nodes. To solve the problem given by Eq. \eqref{eq_lr_cent} locally and in a distributed way, every node takes its own share of data $\boldsymbol{\chi}_i$ and information of neighbouring nodes $j$. Since the data at every node is partial, the parameter values $\boldsymbol{\omega}_i$ and $\nu_i$ are different at different nodes. To reach \textit{consensus} on these values, some key information is shared among the computing nodes. As a result, in the distributed case, the problem~\eqref{eq_lr_cent} takes the following form,
\begin{equation} \label{eq_lr_dist}
\begin{aligned}
	\displaystyle
	\min_{\boldsymbol{\omega}_1,\nu_1,\ldots,\boldsymbol{\omega}_n,\nu_n}
	\quad &  \sum_{i=1}^{n} g_i(\boldsymbol{\omega}_i,\nu_i)= \sum_{i=1}^{n} \sum_{j=1}^{N_i}  (\boldsymbol{\omega}_i^\top \boldsymbol{\chi}_j^i - \nu_i -y_j)^2 \\
	\text{subject to} \quad&  \boldsymbol{\omega}_1 = \dots = \boldsymbol{\omega}_n, \nu_1 = \dots =\nu_n
\end{aligned}
\end{equation}
This problem can be summarized in a compact form as part of problem formulation \eqref{eq_opt} with ML optimization parameter $\mb{y}_i=[\boldsymbol{\omega}_i^\top;\nu_i] \in \mathbb{R}^m$. Note that, in this formulation, $m$ denotes the dimension of parameters describing the hyperplane.
\subsubsection{Support-Vector-Machine}
Support-vector-machine (SVM) is a supervised learning method that classifies a given set of data by finding the best hyperplane that separates all data points into two classes.
Consider~$N$ data points~${\boldsymbol{\chi}_i \in \mathbb{R}^{m-1}}$, ${i=1,\ldots,N}$  labeled by two classes~${l_i \in \{-1,1\}}$. Then, the problem is to find the hyperplane~${\boldsymbol{\omega}^\top \boldsymbol{\chi} - \nu =0}$,  for~${\boldsymbol{\chi}\in\mbb R^{m-1}}$, that partitions the given data into two classes (on two sides of the hyperplane). In the case that the data points are not linearly separable, a nonlinear mapping~$\phi(\cdot)$ is used to map the data into another space where the data can be linearly separated.
The SVM problem, then, is to minimize the following function (known as the Hinge loss~\cite{chapelle2007training,dogan2016unified}):
\begin{equation} \label{eq_svm_cent}
\begin{aligned}
	\displaystyle
	& \min_{[\boldsymbol{\omega}^\top;\nu]}
	~ &  \boldsymbol{\omega}^\top \boldsymbol{\omega} + C \sum_{j=1}^{N} \max\{1-l_j( \boldsymbol{\omega}^\top \phi(\boldsymbol{\chi}_j)-\nu),0\}^p
\end{aligned}
\end{equation}
with~${p \in \mathbb{N}}$ and~$C \in \mathbb{R}^+$ determining the smoothness and margin size, repectively.

In the \textit{distributed} case, the data is distributed among~$n$ computing nodes, each having a partial data set $\boldsymbol{\chi}_i$ with $N_i$ data points. Similar to the previous case (linear regression), the nodes/agents need to reach a consensus on locally found values~$\boldsymbol{\omega}_i$ and~$\nu_i$. Then, the distributed formulation to reach a common classifier is formulated as the following problem,
\begin{equation} \label{eq_svm_dist0}
\begin{aligned}
	\displaystyle
	\min_{\boldsymbol{\omega}_1,\nu_1,\ldots,\boldsymbol{\omega}_n,\nu_n}
	\quad &  \sum_{i=1}^{n} \boldsymbol{\omega}_i^\top \boldsymbol{\omega}_i + C \sum_{j=1}^{N} \max\{\theta_{i,j},0\}^p \\
	\text{subject to} \quad&  \boldsymbol{\omega}_1 = \dots = \boldsymbol{\omega}_n,\qquad\nu_1 = \dots =\nu_n,
\end{aligned}
\end{equation} \normalsize
with~${\theta_{i,j}=1-l_j( \boldsymbol{\omega}_i^\top \phi(\boldsymbol{\chi}^i_j)-\nu_i)}$.
Some works in the literature use another standard approximation of the above Hinge loss model by considering the following objective function for sufficiently large $\mu \in \mathbb{R}^+$ \cite{slp_book,zhang2003modified},
\begin{equation} \label{eq_svm_dist}
\begin{aligned}
	\displaystyle
	\min_{\boldsymbol{\omega}_1,\nu_1,\ldots,\boldsymbol{\omega}_n,\nu_n}
	\quad &  \sum_{i=1}^{n} \boldsymbol{\omega}_i^\top \boldsymbol{\omega}_i + C \sum_{j=1}^{N_i} \tfrac{1}{\mu}\log (1+\exp(\mu \theta_{i,j})) \\
	\text{subject to} \quad&  \boldsymbol{\omega}_1 = \dots = \boldsymbol{\omega}_n,\qquad\nu_1 = \dots =\nu_n,
\end{aligned}
\end{equation} \normalsize
Similarly, one can summarize the problem as part of the formulation \eqref{eq_opt} with ML optimization parameter $\mb{y}_i=[\boldsymbol{\omega}_i^\top;\nu_i] \in \mathbb{R}^m$ and $m$ as the dimension of the hyperplane classifier parameters.

\begin{rem} \label{rem_data}
It is known from the literature \cite{xin2020decentralized,qureshi2021decentralized,mcmahan2017communication} that both the structure of the network and distribution of the data among the nodes affect the convergence rate of the optimization algorithm. For example, exponential networks are known to result in lower optimality gap and faster convergence of distributed optimization. This is because the organized topology of the network, where each node is connected to an increasing number of neighbours at each layer, provides more connectivity, improved information exchange, and reduced communication bottleneck. On the other hand, in the case of \textit{heterogeneous} data distribution among the computing nodes (in contrast to the \textit{homogeneous} case where all data are available at all nodes), the optimality gap is larger and the convergence rate is slower. This is due to imbalanced information processing and data exchange among the nodes. These are also shown by simulation in Section~\ref{sec_sim}.
\end{rem}

\subsection{Algebraic Graph Theory}
The information-sharing among the computing nodes is modelled by a graph topology $\mc{G}_\gamma = \{\mc{V},\mc{E}_\gamma\}$ including a set of $n$ nodes in $\mc{V}$ and links $(i,j) \in \mc{E}_\gamma$ representing the information exchange between nodes $i,j$. The set of neighbours of node $i$, denoted by $\mc{N}^\gamma_i$, includes all nodes $j \in \mc{V}$ which communicate with $i$, i.e., $(j,i) \in \mc{E}_\gamma$.  The network topology among the computing nodes might be time-varying, changing via a switching signal $\gamma: t \mapsto \Gamma$ with $\Gamma$ as the set of all possible topologies for $\mc{G}_\gamma$.
The matrix $W_\gamma = [w_{ij}^\gamma]$ is the weighting adjacency matrix of this information-sharing network $\mc{G}_\gamma$ among the nodes (and depends on $\gamma$). The weight $w_{ij}^\gamma$ implies how node $i$ weights information sent by node $j$. Define the associated Laplacian matrix $\overline{W}_\gamma = [\overline{w}_{ij}^\gamma]$ as
\begin{align} \label{eq_laplac}
\overline{w}_{ij}^\gamma = \left\{
\begin{array}{ll}
	-\sum_{i=1}^{n} w^\gamma_{ij}, & i=j \\
	w^\gamma_{ij}, & i\neq j.
\end{array}\right.
\end{align}

\begin{ass} \label{ass_net}
The network of computing nodes $\mc{G}_\gamma$ is time-varying, connected, and undirected with symmetric weights (i.e., $w_{ij}^\gamma(t)=w_{ji}^\gamma$ for all $t \geq 0$).
\end{ass}

\begin{lem} \label{lem_laplac}
\cite{SensNets:Olfati04,olfatisaberfaxmurray07}
For a network satisfying Assumption~\ref{ass_net}, all the eigenvalues of $\overline{W}_\gamma$ are real-valued and negative, except one isolated zero eigenvalue with left (and right) eigenvector $\mb{1}_n^\top$ (and~$\mb{1}_n$), i.e., $\mb{1}_n^\top \overline{W}_\gamma= \mb{0}_n$ and~$\overline{W}_\gamma \mb{1}_n=\mb{0}_n$.
\end{lem}

\section{The Proposed Networked Algorithm}  \label{sec_alg}
\subsection{Linear Solution: Ideal Data-Exchange} \label{sec_alg_lin}
First, we consider the case that the information exchange among the computing nodes is ideal (not quantized) and, therefore, the solution is linear.
\begin{align} \label{eq_sol_lin}
	\dot{\mb{x}}_i &=  \sum_{j=1}^{n} w_{ij}^\gamma \Big((\partial_{\mb{x}_j} f_j + \partial_{\mb{x}_j} f^\Xi_j) -  (\partial_{\mb{x}_i} f_i + \partial_{\mb{x}_i} f^\Xi_i)\Big), \\ \label{eq_xdot_lin} 
	\dot{\mb{y}}_i &= -\sum_{j=1}^{n} w_{ij}^\gamma (\mb{y}_i-\mb{y}_j)-\alpha \mb{z}_i, \\ \label{eq_ydot_lin}
	\dot{\mb{z}}_i &= -\sum_{j=1}^{n} w_{ij}^\gamma (\mb{z}_i-\mb{z}_j ) + \partial_t \nabla g_i(\mb{y}_i),
\end{align}
where $\alpha \in \mathbb{R}^+$ denotes the gradient-tracking (GT) rate, $w_{ij}^\gamma \in \mathbb{R}^+$ is the weight on the data-sharing link between nodes $i$ and $j$ under the switching signal $\gamma$. The local variable $\mb{z}_i$ is an auxiliary variable for gradient tracking at computing node $i$, and is initially set to zero, i.e., $\mb{z}_i(0)=0$ for all $i$.
From Eq.~\eqref{eq_xdot_lin}-\eqref{eq_ydot_lin} and recalling Assumption~\ref{ass_net}, we have
\begin{eqnarray}
	\sum_{i=1}^n \dot{\mb{z}}_i  
	&=& \sum_{i=1}^n \partial_t \nabla g_i(\mb{y}_i),  \label{eq_sumydot} \\ \label{eq_sumxdot}
	\sum_{i=1}^n \dot{\mb{y}}_i
	&=& -\alpha \sum_{i=1}^n\mb{z}_i.
\end{eqnarray}
By simple integration with respect to~$t$ and recalling that $\mb{z}_i(0)=0$, we get
\begin{eqnarray} \label{eq_sumxdot2}
	\sum_{i=1}^n \dot{\mb{y}}_i = -\alpha \sum_{i=1}^n \mb{z}_i = -\alpha \sum_{i=1}^n \boldsymbol{ \nabla} g_i(\mb{y}_i),
\end{eqnarray}
This illustrates the gradient-tracking nature of the proposed dynamics.

The auxiliary objective functions $f^\Xi_i(\cdot)$ address the local box constraints in \eqref{eq_opt}. These are typically modelled as smooth penalty (or barrier) functions in the following form,
\begin{align} \label{eq_fi_xi}
	f^\Xi_i(\mb{x}_i) = \epsilon([\mb{x}_i - \overline{\xi}]^+ + [\underline{\xi} - \mb{x}_i ]^+)
\end{align}
with $\overline{\xi},\underline{\xi}$ as the upper/lower-bounds on $\mb{x}_i$ and parameter $\epsilon \in \mathbb{R}^+$ as a constant weight on the penalty term as compared to the original objective $f_i(\cdot)$. The smooth function $[u]^+$ is defined as,
\begin{align}
	[u]^+=\max \{u, 0\}^\sigma,~\sigma \in \mathbb{N},~\sigma \geq 2
	\label{eq_sigma}
\end{align}
Another smooth penalizing function is proposed in \cite{nesterov1998introductory},
\begin{align}
	[u]^+=\frac{1}{\sigma}\log(1+\exp(\sigma u)),~\sigma \in \mathbb{R}^+
	\label{eq_sigma2}
\end{align}
where the parameter $\sigma$ is chosen sufficiently large to make the penalty function as close as possible to \eqref{eq_sigma}. Using this penalty formulation, one can decouple the scheduling and ML objectives.
In this direction, using the laplacian matrix definition in Eq.~\eqref{eq_laplac}, we rewrite the dynamics \eqref{eq_xdot_lin}-\eqref{eq_ydot_lin} in compact form as
\begin{align} \label{eq_xydot1}
	\left(\begin{array}{c} \dot{\mb{y}} \\ \dot{\mb{z}} \end{array} \right) = A(t,\alpha,\gamma) \left(\begin{array}{c} {\mb{y}} \\ {\mb{z}} \end{array} \right),
\end{align}
where $A(t,\alpha,\gamma)$ represents the ML system matrix defined as
\begin{align} \label{eq_M}
	\left(\begin{array}{cc} \overline{W}_{\gamma} \otimes I_m & -\alpha I_{mn} \\ H(\overline{W}_{\gamma}\otimes I_m) & \overline{W}_{\gamma} \otimes I_m - \alpha H
	\end{array} \right),
\end{align}
with $H:=\mbox{diag}[\nabla^2 g_i(\mb{y}_i)]$.

\subsection{Nonlinear Solution: Log-Scale Quantized Data-Exchange}
Next, we consider the case that the information exchange among the nodes is logarithmically quantized and, therefore, the solution is nonlinear.
\begin{align} \label{eq_sol_q}
	\dot{\mb{x}}_i &=  \sum_{j=1}^{n} w_{ij}^\gamma \Big(q(\partial_{\mb{x}_j} f_j + \partial_{\mb{x}_j} f^\Xi_j) -  q(\partial_{\mb{x}_i} f_i + \partial_{\mb{x}_i} f^\Xi_i)\Big), \\ \label{eq_xdot_q} 
	\dot{\mb{y}}_i &= -\sum_{j=1}^{n} w_{ij}^\gamma (q(\mb{y}_i)-q(\mb{y}_j))-\alpha \mb{z}_i, \\ \label{eq_ydot_q}
	\dot{\mb{z}}_i &= -\sum_{j=1}^{n} w_{ij}^\gamma (q(\mb{z}_i)-q(\mb{z}_j) ) + \partial_t \nabla g_i(\mb{y}_i),
\end{align}
where $q(\cdot):\mathbb{R} \mapsto \mathbb{R}$ denotes log-scale quantization as a nonlinear mapping of the information exchanged between nodes. This quantization mapping is defined as,
\begin{align}\label{eq_hl_qlog}
	q(x) = \mbox{sgn}(x)\exp\left(\rho\left[\dfrac{\log(|x|)}{\rho}\right] \right),
\end{align}
where $[\cdot]$ denotes rounding to the nearest integer and $\mbox{sgn}(\cdot)$ denotes the sign function. The parameter $\rho$ denotes the quantization level. As the quantization level goes to zero $\rho \rightarrow 0$, the nonlinear solution \eqref{eq_sol_q}-\eqref{eq_ydot_q} converges to the linear case \eqref{eq_sol_lin}-\eqref{eq_ydot_lin}. Note that logarithmic quantization is a sector-bound nonlinearity satisfying,
\begin{align}\label{eq_qlog}
	(1-\frac{\rho}{2})x \leq q(x) \leq (1+\frac{\rho}{2})x
\end{align}
where the quantization level generally satisfies $\rho \leq 1$.
Define the diagonal matrix $Q(t) = \mbox{diag}[\overline{q}_i(t)]$ with $\overline{q}_i(t) = \frac{q(\mb{x}_i)}{\mb{x}_i}$ (element-wise division). Then, we define a new modified laplacian matrix $\overline{W}_{\gamma,q} :=  \overline{W}_{\gamma} Q(t)$. This simply follows from $q(\mb{x}(t)) = Q(t) \mb{x}(t)$. Recall from the consensus nature of the solution that $-\sum_{j=1}^{n} w_{ij}^\gamma (\mb{x}_i-\mb{x}_j)$ can be written as $(\overline{W}_{\gamma} \otimes I_m)\mb{x}$. Then, following Eq.~\eqref{eq_qlog}, one can get $-\sum_{j=1}^{n} w_{ij}^\gamma (q(\mb{x}_i)-q(\mb{x}_j))=(\overline{W}_{\gamma,q} \otimes I_m)\mb{x}$. Then, $A(t,\alpha,\gamma)$ in compact formulation \eqref{eq_xydot1} can be modified as
\begin{align} \label{eq_M_g}
	A(t,\alpha,\gamma) := \left(\begin{array}{cc} \overline{W}_{\gamma,q} \otimes I_m & -\alpha I_{mn} \\ H(\overline{W}_{\gamma,q}\otimes I_m) & \overline{W}_{\gamma,q} \otimes I_m - \alpha H
	\end{array} \right),
\end{align}
Finally, we summarize our proposed problem-solving in Algorithm~\ref{alg_1}. The per-iteration computational complexity of this algorithm is of order $\mc{O}(m^2n^3)$.

\begin{algorithm}
		\textbf{Input:}  $\mc{G}^\gamma$, $W_\gamma$, $\alpha$, $\overline{\xi}$, $\underline{\xi}$, $b$, $f_i(\cdot)$, $g_i(\cdot)$, $\sigma$, $\epsilon$, $\rho$\;
		\textbf{Initialization:} $\sum_{i=1}^n \mb{x}_i(0)=b$, $\mb{z}_i(0) = 0$, random $\mb{y}_i(0)$\;
		\While{termination criteria NOT true}{
			Computing node $i$ receives state information from incoming neighbouring nodes $j \in \mc{N}^\gamma_i$\;
			Node $i$ computes Eq.~\eqref{eq_sol_q}-\eqref{eq_ydot_q}\;
			Computing node $i$ shares its updated states with outgoing neighboring nodes $j$ for which $i \in \mc{N}^\gamma_j$\;
		}
		\textbf{Return} Final state $\mb{x}_i,\mb{y}_i,\mb{z}_i$ and overall cost
		%$\sum_{i=1}^{n} f_i(\mb{x}_i) + f_i^\Xi(\mb{x}_i)+ g_i(\mb{y}_i)$\; 
		\caption{CPU resource scheduling and ML optimization algorithm}
		\label{alg_1}
\end{algorithm}

\section{Analysis of Convergence, Feasibility and Optimality} \label{sec_analysis}
In this section, we prove convergence, optimality, and feasibility for the quantized solution \eqref{eq_sol_q}-\eqref{eq_ydot_q}. These results can be easily extended to the linear case \eqref{eq_sol_lin}-\eqref{eq_ydot_lin}. First, we show that the proposed solution is all-time feasible. This implies that at all times the resource-demand balance denoted by the constraint $\sum_{i=1}^{n} \mb{x}_i = b$ holds.
\begin{lem} \label{lem_feas}
	(\textbf{all-time feasibility})
	Given that Assumption~\ref{ass_net} holds on the network connectivity and the computing resources are initially feasible (i.e., $\sum_{i=1}^{n} \mb{x}_i(0) = b$), the proposed solution \eqref{eq_sol_q}-\eqref{eq_ydot_q} is all-time feasible.
\end{lem}
\begin{proof}
	We prove that the change in the computing resources under the proposed dynamics~\eqref{eq_sol_q}-\eqref{eq_ydot_q} is zero, i.e., $\sum_{i=1}^n \dot{\mb{x}}_i = 0$. Recall from Eq.~\eqref{eq_sol_q} that,
	\begin{align} \nonumber
		\sum_{i=1}^n \dot{\mb{x}}_i = \sum_{i=1}^n \sum_{j=1}^{n} w_{ij}^\gamma \Big(&q(\partial_{\mb{x}_j} f_j + \partial_{\mb{x}_j} f^\Xi_j) \\&-  q(\partial_{\mb{x}_i} f_i + \partial_{\mb{x}_i} f^\Xi_i)\Big).  \label{eq_feas_proof}
	\end{align}
	From Assumption \ref{ass_net} we have $w_{ij}^\gamma(t)=w_{ji}^\gamma$; and from the definition \eqref{eq_hl_qlog} it is clear that log-scale quantization is a sign-preserving and odd mapping. Therefore,
	\begin{align} \nonumber
		w_{ij}^\gamma \Big(&q(\partial_{\mb{x}_j} f_j + \partial_{\mb{x}_j} f^\Xi_j) -  q(\partial_{\mb{x}_i} f_i + \partial_{\mb{x}_i} f^\Xi_i)\Big) =\\
		&-w_{ji}^\gamma \Big(q(\partial_{\mb{x}_i} f_i + \partial_{\mb{x}_i} f^\Xi_i) -  q(\partial_{\mb{x}_j} f_j + \partial_{\mb{x}_j} f^\Xi_j)\Big).
	\end{align}
	This implies that the summation in Eq. \eqref{eq_feas_proof} over all $i,j$ gives zero. As a result, by initializing from a feasible solution $\sum_{i=1}^{n} \mb{x}_i(0) = b$, we get
	\begin{align} \nonumber
		\sum_{i=1}^{n} \mb{x}_i(t)  = \sum_{i=1}^{n} \mb{x}_i(0) = b.
	\end{align}
	This proves the lemma.
\end{proof}

	\begin{rem}
		Lemma~\ref{lem_feas} proves the all-time feasibility of the computing resource-demand balance in this work, which implies that the algorithm can be terminated at any time with no constraint violation. This is in contrast to \cite{grammenos2023cpu,rikos2021optimal}, where no all-time constraint feasibility is given. In other words, the algorithms in \cite{grammenos2023cpu,rikos2021optimal} reach the computing resource-demand balance at the convergence time, and before that time, there is no balance between the assigned resources and computing demand.
\end{rem}

The next lemma describes the main feature of the optimal point $[\mb{x}^*;\mb{y}^*]$. First, define
$$ F(\mb{x}) =\sum_{i=1}^n f_i(\mb{x}_i)+f^\Xi_i{\mb{x}_i},$$
$$\nabla_\mb{x} F = [\partial_{\mb{x}_1} f_1 + \partial_{\mb{x}_1} f^\Xi_1;\dots;\partial_{\mb{x}_n} f_n + \partial_{\mb{x}_n} f^\Xi_n].$$
\begin{lem} \label{lem_z*}
	(\textbf{optimality}) Given the problem \eqref{eq_opt}, the optimal state $\mb{x}^*$ satisfies $\nabla_\mb{x} F(\mb{x}^*) \in \mbox{span}(\mb{1}_n)$ and $\mb{y}^* \in \mbox{span}(\mb{1}_n)$. Moreover, this optimal point is invariant under the solution dynamics~\eqref{eq_sol_q}-\eqref{eq_ydot_q}.
\end{lem}
\begin{proof}
	The proof for $\mb{y}^*$ directly follows the consensus constraint $\mb{y}_1 = \mb{y}_2 = \dots = \mb{y}_n$.
	The proof for $\mb{x}^*$ is a result of the Karush–Kuhn–Tucker condition for linear constraint $\sum_{i=1}^{n} \mb{x}_i = b$ and convex objective function given by Eq.~\eqref{eq_opt} and \eqref{eq_fi_xi}. For more details on the latter, refer to \cite{Boyd-CVXBook,bertsekas_lecture}.
\end{proof}
%the following uniquely holds at~${\mb{x}=\mb{x}^*=\mb{1}_n \otimes \overline{ \mb{x}}^*}$,
%\[\sum_{i=1}^n \dot{\mb{x}}_i = -\alpha (\mathbf 1_n^\top \otimes I_m) \boldsymbol{ \nabla} F(\mb{x}^*) =  \mb{0}_m.
%\]
%Further, from~\eqref{eq_xdot} we have~$\dot{\mb{x}}_i = \mb{0}_m$
%and from~\eqref{eq_ydot}, % and~\eqref{eq_dtdf},
%\[\dot{\mb{y}}_i = \frac{d}{d t} \boldsymbol{ \nabla} f_i(\overline{ \mb{x}}^*)=  \boldsymbol{ \nabla}^2 f_i(\overline{ \mb{x}}^*) \dot{\mb{x}}_i = \mb{0}_m,
%\]
%which shows that~$[\mb{x}^*;\mb{0}_{nm}]$ is an invariant equilibrium point of the dynamics~\eqref{eq_xdot}-\eqref{eq_ydot}.

Note that, for the auxiliary variable $\mb{z} \rightarrow \mb{z}^*=\mb{0}_n$.
To prove the convergence, we first recall some useful lemmas.
\begin{lem} \label{lem_sum}
	For any $\mb{x} \in \mathbb{R}^n$ and log-scale quantization $q(\cdot)$, under Assumption~\ref{ass_net} we have
	
	\small \begin{align} \label{eq_sum_lem}
		\sum_{i=1}^n \mb{x}_i \sum_{j=1}^n w_{ij}^\gamma (q(\mb{x}_j)-q(\mb{x}_i)) = \sum_{i,j=1}^n \frac{w_{ij}^\gamma}{2} (\mb{x}_j-\mb{x}_i) q(\mb{x}_j)-q(\mb{x}_i)
	\end{align} \normalsize
\end{lem}
\begin{proof}
	From Assumption~\ref{ass_net}, $w_{ij}^\gamma=w_{ji}^\gamma$. Also, log-scale quantization is an odd and sign-preserving mapping. These imply that,
	\begin{align} \nonumber
		\mb{x}_i w_{ij}^\gamma (q(\mb{x}_i)&-q(\mb{x}_j)) + \mb{x}_j w_{ji}^\gamma (q(\mb{x}_j)-q(\mb{x}_i)) \\ \nonumber
		& = w_{ij}^\gamma (\mb{x}_i-\mb{x}_j)(q(\mb{x}_j)-q(\mb{x}_i)) \\
		& = w_{ji}^\gamma (\mb{x}_j-\mb{x}_i) (q(\mb{x}_j)-q(\mb{x}_i)).
	\end{align}
	and the proof follows.
\end{proof}

\begin{lem} \label{lem_dM}
	\cite{stewart_book,cai2012average} Consider an $n$-by-$n$ matrix~$P(\alpha)$ as a function of~${\alpha \geq 0}$. Assume that this matrix has~${l<n}$ equal eigenvalues~$\lambda_1=\ldots=\lambda_l$ with (right and left) linearly independent unit eigenvectors~$\mb{v}_1,\ldots,\mb{v}_l$ and~$\mb{u}_1,\ldots,\mb{u}_l$.
	Then, $\frac{d\lambda_i}{d\alpha}|_{\alpha=0}$ is the $i$-th eigenvalue of the following matrix:
	\begin{align}
		\left(\begin{array}{ccc}
			\mb{u}_1^\top P' \mb{v}_1 & \ldots & \mb{u}_1^\top P' \mb{v}_l \\
			& \ddots & \\
			\mb{u}_l^\top P' \mb{v}_1 & \ldots & \mb{u}_l^\top P' \mb{v}_l
		\end{array} \right), ~ P' =  \frac{dP(\alpha)}{d\alpha}|_{\alpha=0}
	\end{align}
\end{lem}

\begin{lem} \label{lem_zeroeig}
	Given the compact system dynamics~\eqref{eq_xydot1} with $A(t,\alpha,\gamma)$ as Eq. \eqref{eq_M_g}, the system has all negative eigenvalues except $m$ zero eigenvalues for sufficiently small $\alpha$, with $m$ as the dimension of the ML optimization variable.
\end{lem}
\begin{proof}
	First, rewrite Eq. \eqref{eq_M_g} as
	\begin{align}  \label{eq_Mg}
		A(t,\alpha,\gamma) &=   A_q^0 + \alpha A^1, ~ A_q^0 = Q(t)  A^0 \\
		(1-\frac{\rho}{2}) A^0 & \preceq A_q^0 \preceq (1+\frac{\rho}{2}) A^0 \\ \label{eq_beta_M}
		(1-\frac{\rho}{2}) I_n  &\preceq Q(t) \preceq (1+\frac{\rho}{2}) I_n
	\end{align}
	where
	\begin{eqnarray}\nonumber
		A^0 &=&   \left(\begin{array}{cc} \overline{W}_\gamma \otimes I_m & \mb{0}_{mn\times mn} \\ H(\overline{W}_\gamma \otimes I_m) & \overline{W}_\gamma \otimes I_m \end{array} \right),\\\nonumber
		A^1 &=& \left(\begin{array}{cc} \mb{0}_{mn\times mn} & - {I_{mn}} \\ {\mb{0}_{mn\times mn}} & - H \end{array} \right).
	\end{eqnarray}
	One can write $ (1-\frac{\rho}{2}) \sigma(A^0) \leq \sigma(A^0_q) \leq (1+\frac{\rho}{2}) \sigma(A^0)$, where $\sigma(A^0) = \sigma(\overline{W} \otimes I_m) \cup \sigma(\overline{W} \otimes I_m)$. Then, from Lemma~\ref{lem_laplac}, $A^0$ has~$m$ set of eigenvalues which satisfy
	$$\lambda_{2n,j} \leq \ldots \leq \lambda_{3,j} < \lambda_{2,j} = \lambda_{1,j} = 0,~j=\{1,\ldots,m\}.$$
	Then, using Lemma~\ref{lem_dM}, one can check how the perturbation~$\alpha A^1$ affects $\lambda_{1,j}$ and~$\lambda_{2,j}$ in $\sigma(A_q^0)$.
	%Let~$\lambda_{1,j}(\alpha,t)$ and~$\lambda_{2,j}(\alpha,t)$ denote  the perturbed eigenvalues by~$\alpha M^1$ and follows \eqref{eq_spect_k}.
	Following Lemma~\ref{lem_dM},
	\begin{align} \nonumber
		\mb{v} = [\mb{v}_1~\mb{v}_2] =\left(\begin{array}{cc}
			\mb{1}_n& \mb{0}_n \\
			\mb{0}_n & \mb{1}_n
		\end{array} \right)\otimes I_m,
		% \frac{1}{\sqrt{m}}
	\end{align}
	and $\mb{u}=\mb{v}^\top$. From \eqref{eq_M_g}, we have~$\frac{dA(\alpha)}{d\alpha}|_{\alpha=0}=A^1$. Then,
	\begin{eqnarray} \label{eq_dmalpha}
		\mb{v}^\top A^1 \mb{v}= \left(\begin{array}{cc}
			\mb{0}_{m\times m} & \mb{0}_{m\times m} \\
			... & -(\mb{1}_n \otimes I_m)^\top H (\mb{1}_n \otimes I_m)
		\end{array} \right),
	\end{eqnarray}
	which has $m$ zero eigenvalues and $m$ other eigenvalues are negative. This follows Assumption~\ref{ass_cost} since
	\begin{equation} \label{eq_sum_df}
		-(\mb{1}_n \otimes I_m)^\top H  ( \mb{1}_n \otimes I_m)= -\sum_{i=1}^n \nabla^2  f_i(\mb{x}_i) \prec 0,
	\end{equation}
	Therefore,~${\frac{d\lambda_{1,j}}{d\alpha}|_{\alpha=0} = 0}$ and~${\frac{d\lambda_{2,j}}{d\alpha}}|_{\alpha=0}<0$. This implies that after perturbation of $A^0_q$ by~$\alpha A^1$, $m$ zero eigenvalues~$\lambda_{2,j}(\alpha,t)$ become negative while  $\lambda_{1,j}(\alpha,t)$ remain zero for all $j$. Further, we need to show that for sufficiently small $\alpha$ other negative eigenvalues still remain in the left-half-plane.
	First, we recall from \cite[Appendix]{delay_est} to relate the spectrum of $A(t,\alpha,\gamma)$ in \eqref{eq_Mg} to $\alpha$. For simplicity of the proof analysis, from this point onward we set $m=1$, but the proof holds for any $m>1$. By proper row/column permutations, following \cite[Eq.~(18)]{delay_est},  $\sigma(A(t,\alpha,\gamma))$ can be defined as, 
		\begin{align} \nonumber
			\mbox{det}(\alpha  I_{n}) \mbox{det}(H(\overline{W}_\gamma) +(\overline{W}_\gamma - \alpha H -\lambda I_{n}) (\frac{1}{\alpha})(\overline{W}_\gamma  -\lambda I_{n})) = 0.
		\end{align} 
		This can be simplified as follows,
		\begin{align} \label{eq_m=1}
			\mbox{det}(I_{n}) \mbox{det}((\overline{W}_\gamma  -\lambda I_{n})(\overline{W}_\gamma  - \lambda I_{n}) +\alpha \lambda H ) = 0
		\end{align}
		For stability/convergence, we need to find the admissible $\alpha$ values such that $\lambda$ remains in the left-half-plane, except for one zero eigenvalue\footnote{The eigenspectrum analysis in this section follows the fact that the eigenvalues are continuous functions of the matrix elements~\cite{stewart_book}.}. 
		Clearly, $\alpha=0$ satisfies \eqref{eq_m=1} and leads to the eigen-spectrum following as $\sigma(A(t,\alpha,\gamma)) = \sigma(\overline{W}_\gamma) \cup \sigma(\overline{W}_\gamma)$ which has two zero roots for all switching topologies $\gamma$. The other root $\overline{\alpha}>0$ needs to be defined for the admissible range $0<\alpha<\overline{\alpha}$ for the stability of $A(t,\alpha,\gamma)$. 
		For any $\lambda<0$, with some abuse of notation, Eq. \eqref{eq_m=1} can be reformulated as
		\begin{align} \nonumber
			\mbox{det}((\overline{W}_\gamma  -\lambda I_{n} \pm \sqrt{\alpha |\lambda| H})(\overline{W}_\gamma  - \lambda I_{n} \mp \sqrt{\alpha |\lambda| H} ) ) = 0 \label{eq_det_all}
		\end{align}
		Then, we have,
		\begin{align} \nonumber
			\mbox{det}(\overline{W}_\gamma  - \lambda I_{n} \mp \sqrt{\alpha |\lambda| H} )= \mbox{det}(\overline{W}_\gamma  -\lambda (I_{n}   \mp \sqrt{\frac{\alpha  H}{|\lambda|}} )) = 0
		\end{align}
		This gives $\lambda (1 \pm \sqrt{\frac{\alpha  H}{|\lambda|}})$ as perturbed eigenvalue of $\lambda \in \sigma(\overline{W}_\gamma)$. Therefore, for $\lambda \neq 0$, the minimum $\overline{\alpha}$ making this term zero is,
		\begin{align}
			\overline{\alpha} = \argmin_{\alpha} |1 - \sqrt{\frac{\alpha  H}{|\lambda|}}|
			\geq \frac{\min \{|\lambda|\neq 0\}}{\max \{H_{ii}\}} = \frac{|\lambda_2|}{L}
		\end{align}
		where $H \preceq L I_{n}$ is recalled from Assumption~\ref{ass_cost}.
		Therefore, the admissible range for the convergence of the linear dynamics~\eqref{eq_sol_lin}-\eqref{eq_ydot_lin} is
		\begin{align} \label{eq_alphabar0}
			0 < \alpha < \overline{\alpha}:= \frac{|\lambda_2|}{L}
		\end{align}
		Considering the log-quantization nonlinearity in dynamics~\eqref{eq_sol_q}-\eqref{eq_ydot_q} and from Eq.~\eqref{eq_qlog}, we have
		\begin{align} \label{eq_alphabar00}
			0 < \alpha < \overline{\alpha}:= \frac{|\lambda_2|}{L(1+\frac{\rho}{2})}
		\end{align}
		This gives the admissible range of $\alpha$ for which all the eigenvalues are in the left-half-plane (except one set of $m$ zero eigenvalues). This completes the proof.
\end{proof}

\begin{theorem} \label{thm_converg}
	(\textbf{convergence})
	Suppose Assumptions~\ref{ass_cost}-\ref{ass_net} hold. Then, under initial conditions $\sum_{i=1}^{n} \mb{x}_i(0) = b$, $\mb{z}(0) = \mb{0}_{nm}$ and sufficiently small $\alpha$ satisfying $0 < \alpha < \overline{\alpha}:= \frac{|\lambda_2|}{L(1+\frac{\rho}{2})}$, Algorithm~\ref{alg_1} solves the problem \eqref{eq_opt}.
\end{theorem}
\begin{proof}
	We need to show that the algorithm converges to the optimal point given by Lemma~\ref{lem_z*}. We prove this by defining the proper residual-based Lyapunov function.
	Define the state residual vector for the ML objective as
	\begin{align} \label{eq_delta}
		\delta_m = \left(\begin{array}{c} {\mb{y}} \\ {\mb{z}} \end{array} \right) - \left(\begin{array}{c} {\mb{y}^*} \\ {\mb{0}_{nm}} \end{array} \right)
	\end{align}
	and the residual for scheduling as
	\begin{align} \label{eq_delta2}
		\delta_c = F(\mb{x}) - F(\mb{x}^*)
	\end{align}
	For Lyapunov analysis, we consider the following function,
	\begin{align} \label{eq_V}
		V =   \frac{1}{2}\lVert \delta_m \rVert_2^2 + \delta_c
	\end{align}
	This function is positive-definite with $V \rightarrow 0$ as $\delta \rightarrow \mb{0}$. We have
	\begin{align} \label{eq_Vdot}
		\dot{V} = \delta_m^\top \dot{\delta}_m + \dot{\delta}_c =  \delta_m^\top A {\delta_m}+  \nabla_\mb{x} F \dot{\mb{x}}.
	\end{align}
	Then,  from~\cite[Sections~VIII-IX]{SensNets:Olfati04}, for the first term we get
	\begin{eqnarray} \label{eq_Re2}
		\delta_m^\top A {\delta_m} \leq \lambda_{2} \delta_m^\top {\delta_m},
	\end{eqnarray}
	From Lemma~\ref{lem_zeroeig} we know that ${\lambda}_{2}$ is negative, which implies that this term in $\dot{\mc{V}}$ is negative-definite for $\delta_m \neq \mb{0}$.
	
	For the second term in \eqref{eq_Vdot} we have
	\begin{align}
		\nabla_\mb{x} F \dot{\mb{x}} &= \sum_{i=1}^n  \dot{\mb{x}}_i (\partial_{\mb{x}_i} f_i + \partial_{\mb{x}_i} f^\Xi_i). \end{align}
	Replacing $\dot{\mb{x}}_i$ from Eq.~\eqref{eq_sol_q}, we obtain
	
	\small \begin{align} \nonumber
		\nabla_\mb{x} F \dot{\mb{x}} = \sum_{i =1}^n &(\partial_{\mb{x}_i} f_i + \partial_{\mb{x}_i} f^\Xi_i) \\ &\sum_{j=1}^{n} w_{ij}^\gamma \Big(q(\partial_{\mb{x}_j} f_j + \partial_{\mb{x}_j} f^\Xi_j) -  q(\partial_{\mb{x}_i} f_i + \partial_{\mb{x}_i} f^\Xi_i)\Big).
	\end{align} \normalsize
	Using Lemma~\ref{lem_sum} and Assumption~\ref{ass_net}, we see that
	\begin{align} \nonumber
		\nabla_\mb{x} F \dot{\mb{x}} =  -\sum_{i,j =1}^n \frac{w_{ij}^\gamma}{2} \Bigl( (\partial_{\mb{x}_i} f_i + \partial_{\mb{x}_i} f^\Xi_i) -(\partial_{\mb{x}_j} f_j + \partial_{\mb{x}_j} f^\Xi_j)\Big) \\ \Bigl( q(\partial_{\mb{x}_i} f_i + \partial_{\mb{x}_i} f^\Xi_i) -q(\partial_{\mb{x}_j} f_j + \partial_{\mb{x}_j} f^\Xi_j)\Big)\Bigr).
	\end{align}
	Recall that log-scale quantization is odd and monotonically non-decreasing. This implies that $\nabla_\mb{x} F \dot{\mb{x}} <0$ for $\delta_c \neq \mb{0}$ and $\dot{V}$ is negative-definite. 
	Then, recalling Lyapunov stability theorem \cite{nonlin}, the residual $\delta_m+\delta_c$ under the dynamics~\eqref{eq_sol_q}-\eqref{eq_ydot_q} is decreasing toward zero and, therefore, $F(\mb{x}) \rightarrow F(\mb{x}^*)$ (and $\mb{x} \rightarrow \mb{x}^*$), $\mb{y} \rightarrow \mb{y}^* $, and $\mb{z} \rightarrow \mb{0}_{nm}$.
	This completes the proof.
\end{proof}

	It should be noted that the proof of Lemma~\ref{lem_zeroeig} and Theorem~\ref{thm_converg} follows from Assumption~\ref{ass_cost} and is only based on the constraint $(\mb{1}_n \otimes I_m)^\top H (\mb{1}_n \otimes I_m) \succ 0$ with no other assumption on the convexity of the $g_i(\cdot)$ function. Therefore, the convergence holds for possibly non-convex $g_i(\cdot)$ satisfying Assumption~\ref{ass_cost}. This is shown later via simulations in Section~\ref{sec_sim1}.

\begin{rem}
		The proof analysis for convergence and feasibility in Lemma~\ref{lem_feas},~\ref{lem_zeroeig}, and Theorem~\ref{thm_converg} is irrespective of the structure of the network topology and only requires the network to be connected. Therefore, the network could be time-varying, and the proofs of Lemma~\ref{lem_feas},~\ref{lem_zeroeig}, and Theorem~\ref{thm_converg} still hold. This is in contrast to \cite{grammenos2023cpu,rikos2021optimal}, which consider time-invariant (static) network topology.
\end{rem}

\begin{rem}
		Theorem~\ref{thm_converg} clearly proves convergence subject to log-scale quantization on data-sharing considered in the nodes' dynamics \eqref{eq_sol_q}-\eqref{eq_ydot_q}. This is in contrast to uniform quantization, which in general may result in some optimality gap (or steady-state residual) which is often proportional to the quantization level, e.g., see existing literature on distributed learning/optimization with \textit{uniformly quantized} data \cite{rikos2021optimal,zhu2016quantized,danaee2021energy,hanna2021quantization,bastianello2023online}. This is because log-scale quantization is a sector-bound nonlinearity, while uniform quantization is not sector-bounded (with some bias at the origin), see Fig.~\ref{fig_quant}. In fact, logarithmic quantization assigns more bits to represent smaller values and fewer bits to larger values; this is advantageous for gradient-tracking as smaller gradient values near the optimal point (and more
		critical for convergence) are represented with higher precision. This may increase the complexity of the distributed optimization setup, but reduces the optimality gap.
\end{rem}

\begin{figure}[]
	\centering
	\includegraphics[width=4.3in]{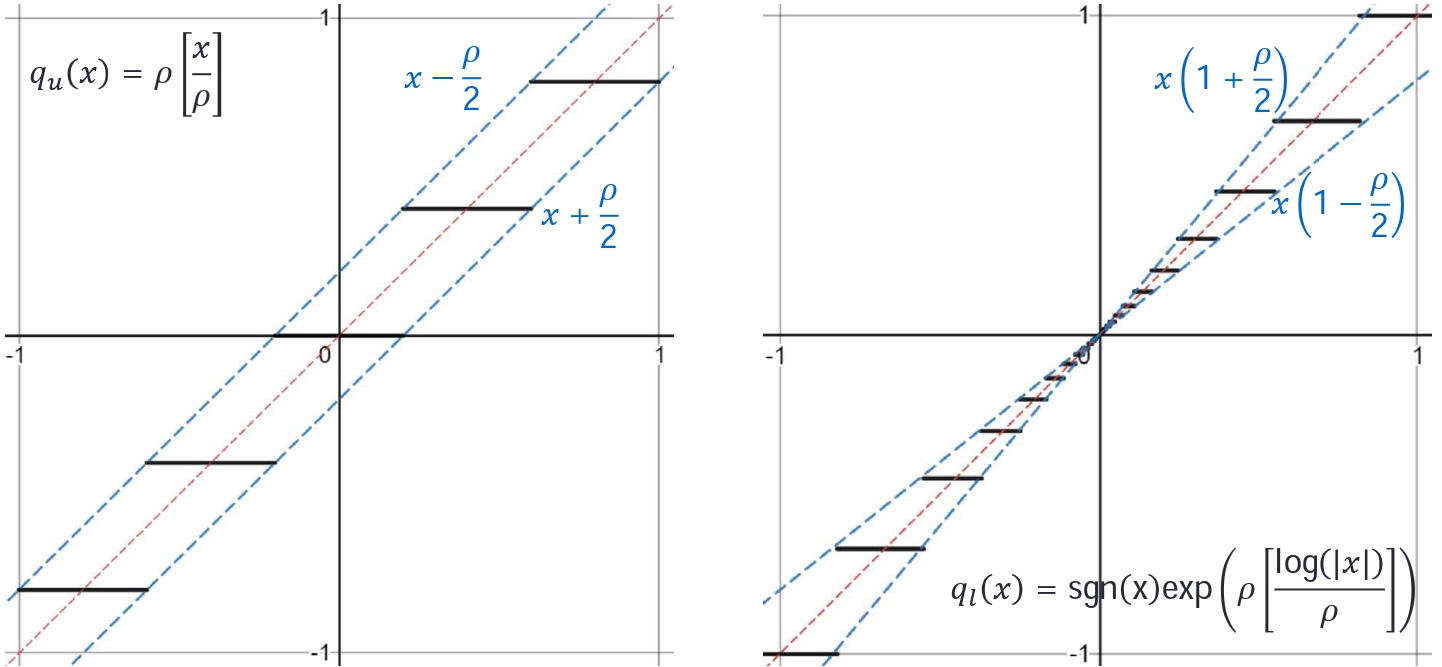} 
	\caption{Comparison between uniform quantization (left) and log-scale quantization (right): uniform quantization is not a sector-bound nonlinearity, while logarithmic quantization is sector-bounded, assigning more bits to represent smaller values and fewer bits to larger values. }
	\label{fig_quant}
\end{figure}

\section{Simulations}\label{sec_sim}
The simulations are performed in an Intel Core-i5 laptop with $2.50$ GHz CPU and $8$ GB RAM. The simulations in Section~\ref{sec_sim1} are performed in MATLAB R$2022$, and simulations in Section~\ref{sec_sim2} are performed in Python.  
\subsection{Academic Examples}\label{sec_sim1}
In contrast to traditional centralized computing, in this section, we apply parallel data processing of data over many computing nodes while cost-optimally assigning the computing resources.
For simulation, we consider a network of $n=20$ computing nodes communicating over a time-varying Erdos-Renyi graph topology with $40\%$ linking probability. The link weights $w_{ij}^\gamma$ are randomly set in the range $(0,1]$. First, we compare the optimal cost of the proposed CPU scheduling algorithm with the CPU balancing strategy in \cite{rikos2021optimal} in Table~\ref{tab_comp}. As stated in Remark~\ref{rem_cpu}, considering our sum-preserving constraint results in less resource allocation cost as compared to the consensus-based balancing strategy in \cite{rikos2021optimal}.

\begin{table} [h]
	\centering
	\caption{Comparing the optimality performance with uniformly-quantized algorithm  \cite{rikos2021optimal} in terms of the cost of assigned CPU resources for different values of overall demand $b$ }
	\label{tab_comp}
	\begin{tabular}{|c|c|c|c|c|}
		\hline
		The demand & $b = 1000$  & $b = 3000$ & $b = 5000$ & $b = 8600$ \\
		\hline
		This work & 84 & 288 & 338 & 417   \\
		\hline
		Ref. \cite{rikos2021optimal} & 1387 & 2015 & 2386 &  2453 \\
		\hline
		\hline
	\end{tabular}
\end{table}

For the rest of this section, we set overall resources $b=8600$ to meet the workload demand and the box constraints as $0<\mb{x}_i<700$ (i.e., $\underline{\zeta}=0$, $\overline{\zeta}=700$ in Algorithm~\ref{alg_1}).
We use these computing resources to solve two ML problems as follows.

\textbf{Distributed SVM:} Consider $N=1000$ data points in 2D as shown in Fig.~\ref{fig_data} to be classified by the SVM line shown in the figure.
\begin{figure}[]
	\centering
	\includegraphics[width=2.5in]{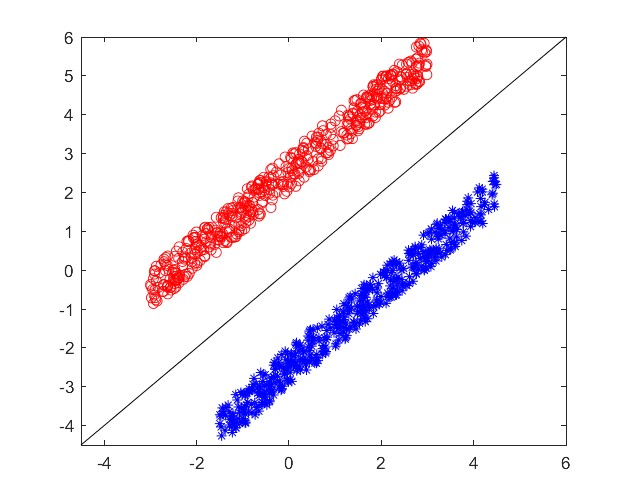} 
	\caption{The data points in 2D used for the SVM classification and the associated SVM classifier line.}
	\label{fig_data}
\end{figure}
The data is distributed among the computing nodes such that every node has access to $75\%$ of the data points. The parameters in Algorithm~\ref{alg_1} are set as $\alpha=0.05$, $\mu = 2$, $C=5$, $\sigma=2$, $\epsilon = 1$, and $\rho = 0.125$. Every computing node minimizes its SVM loss function described in \eqref{eq_svm_dist} and CPU resource allocation cost \eqref{eq:fiz2} to find the optimal variables $\mb{x}_i$, $\omega_i$, and $\nu_i$. The time-evolution of the parameter $\mb{x}_i$ denoting the assigned CPU resources, the SVM classifier parameters $\omega_i \in \mathbb{R}^2,\nu_i \in \mathbb{R}$, and the overall optimality gap (or cost residual) are shown in Fig.~\ref{fig_svm}. As it is clear from the figure, the average value of the assigned resources $\mb{x}_i$ is constant over time, implying resource-demand feasibility at all times. The SVM parameters (both dimensions of vector $\omega_i$ and scalar $\nu_i$) at all nodes reach consensus. The overall residual function $\sum_{i=1}^{n} (f_i(\mb{x}_i) + f_i^\Xi(\mb{x}_i)+ g_i(\mb{y}_i)) - (f_i(\mb{x}^*_i) +g_i(\mb{y}^*_i))$ is decreasing over time; however, the convergence rate is slow since the data is distributed \textit{heterogeneously} among the computing nodes (see Remark~\ref{rem_data}) and only a portion of data is available at each computing node.
\begin{figure}[]
	\centering
	\includegraphics[width=2.5in]{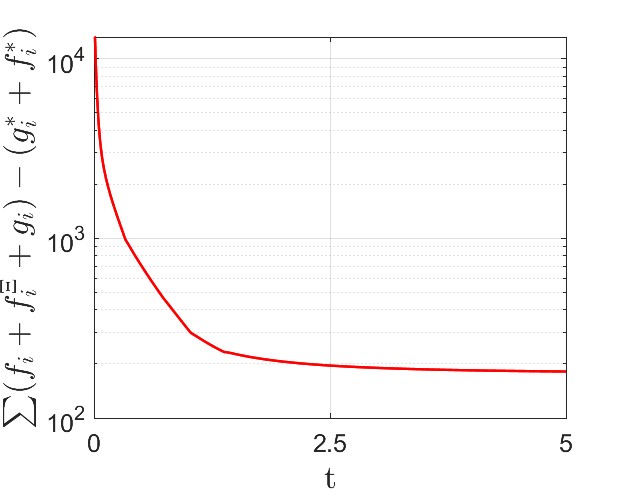}
	\includegraphics[width=2.5in]{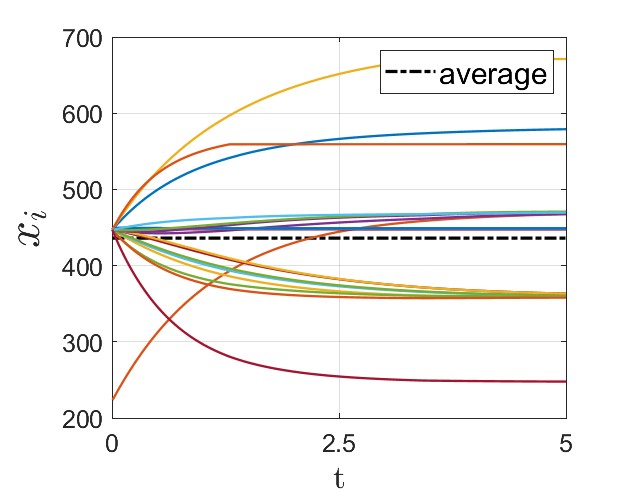}
	\includegraphics[width=2.5in]{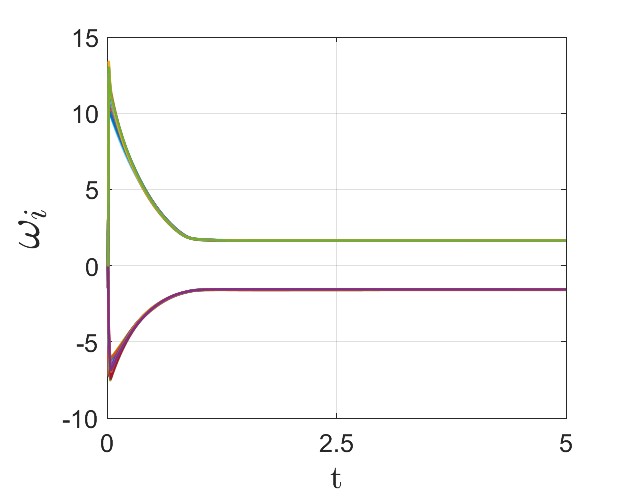} 
	\includegraphics[width=2.5in]{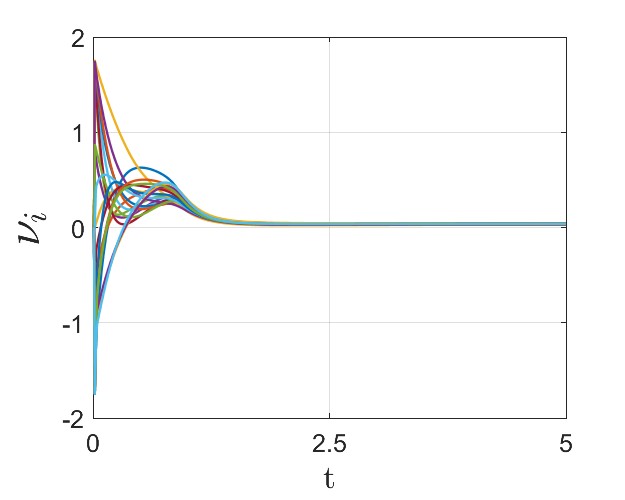} 
	\caption{Time-evolution of the overall cost residual (optimality gap), assigned CPU resources $\mb{x}_i$, and SVM parameters ${\omega}_i,\nu_i$ under the proposed Algorithm~\ref{alg_1}.}
	\label{fig_svm}
\end{figure}

\textbf{Distributed linear regression:}
We consider a set of $N=1000$ randomly generated 2D data points to be fitted with a regressor line. Each computing node has access to $70\%$ of the data points and locally finds the parameters of the regressor line; then, it shares this information with the neighbouring nodes over the network to reach a consensus on the regression parameters.  The parameters in Algorithm~\ref{alg_1} are set as $\alpha=0.1$, $\epsilon=1$,  $\sigma=2$, and $\rho = 0.0625$. Every computing node minimizes its regression loss function described in \eqref{eq_lr_dist} and CPU resource allocation cost \eqref{eq:fiz2} to find the optimal variables $\mb{x}_i$, $\omega_i$, and $\nu_i$. The time-evolution of the CPU parameter $\mb{x}_i$, the regression line parameters $\omega_i ,\nu_i$, and the cost residual are shown in Fig.~\ref{fig_reg}. Clearly, the residual is decreasing toward the optimal point, resource-demand feasibility of CPU resources $\mb{x}_i$ holds at all times, and the regressor line parameters reach consensus at all nodes.
\begin{figure}[]
	\centering
	\includegraphics[width=2.5in]{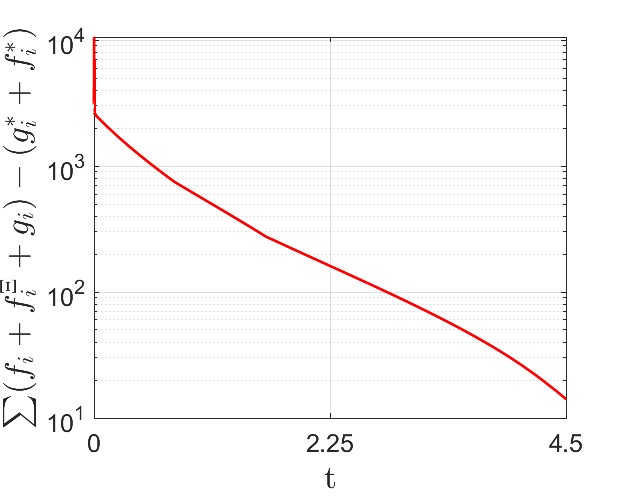}
	\includegraphics[width=2.5in]{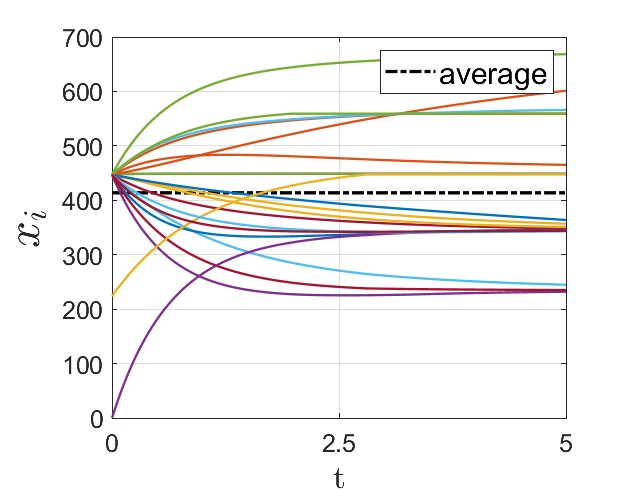}
	\includegraphics[width=2.5in]{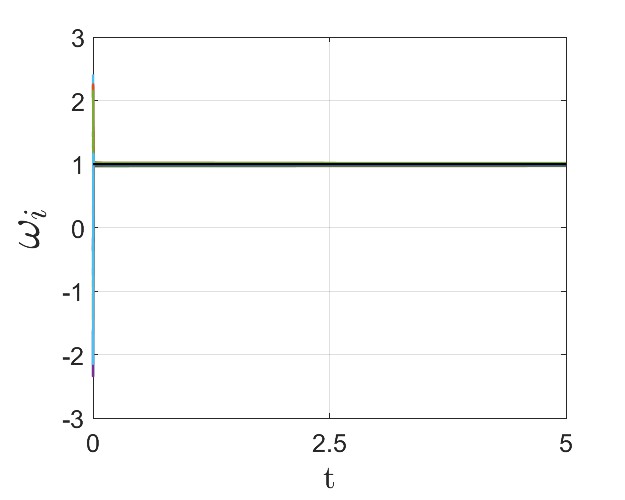} 
	\includegraphics[width=2.5in]{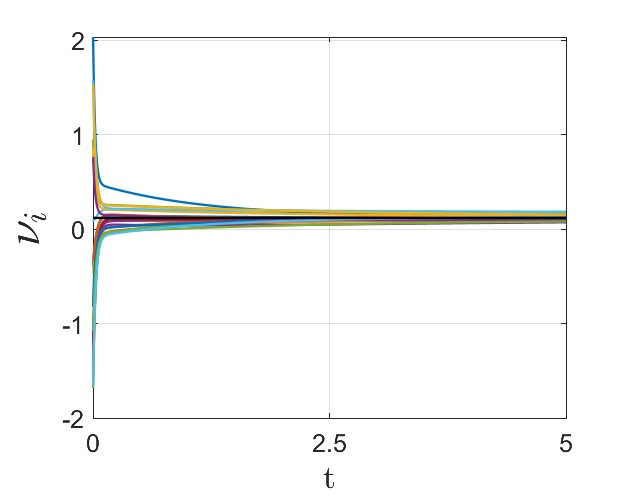} 
	\caption{Time-evolution of the overall cost residual (optimality gap), assigned CPU resources $\mb{x}_i$, and regression parameters $\beta_i,\nu_i$ under the proposed Algorithm~\ref{alg_1}.}
	\label{fig_reg}
\end{figure}

	\textbf{Non-convex objective:} 
	Next, the simulation is performed for a synthetic non-convex local objective function (taken from \cite{xin2021fast}) defined as:
	\begin{align}\label{eq_fij_sim}
		g_{i,j}(y_i) = 2 y_i^2 +3\sin^2(y_i)+a_{i,j} \cos(y_i) + b_{i,j}y_i,
	\end{align} 
	with $\sum_{i=1}^n \sum_{j=1}^N a_{i,j} = 0$ and $\sum_{i=1}^n \sum_{j=1}^N b_{i,j}=0$ such that $a_{i,j},b_{i,j} \neq 0$ and randomly chosen in the range $(-5,5)$ with $N=60$ sample data points. As illustrated in Fig.~\ref{fig_nonconv}, these local objectives are not convex, i.e., $\nabla^2 f_{i}(\mb{x})$ is negative at some points, while the global objective satisfies Assumption~\ref{ass_cost}. 
	\begin{figure}
		\centering
		\includegraphics[width=1.75in]{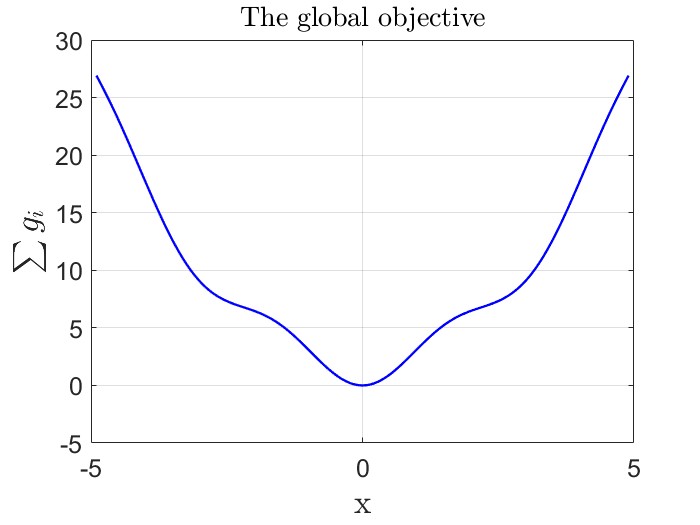}
		\includegraphics[width=1.75in]{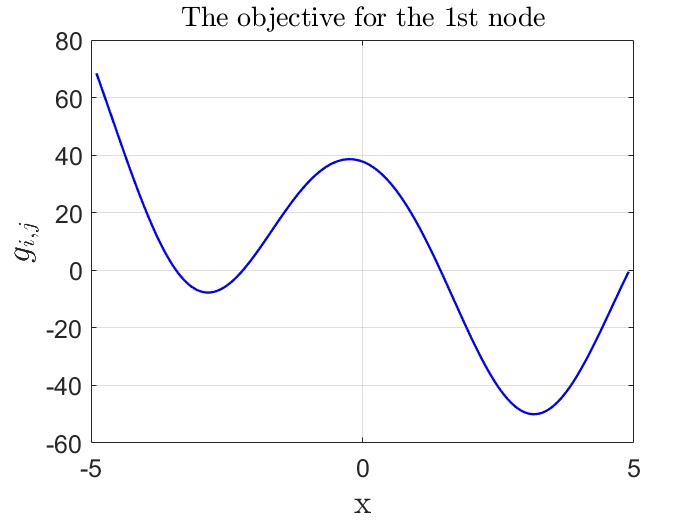}
		\includegraphics[width=1.75in]{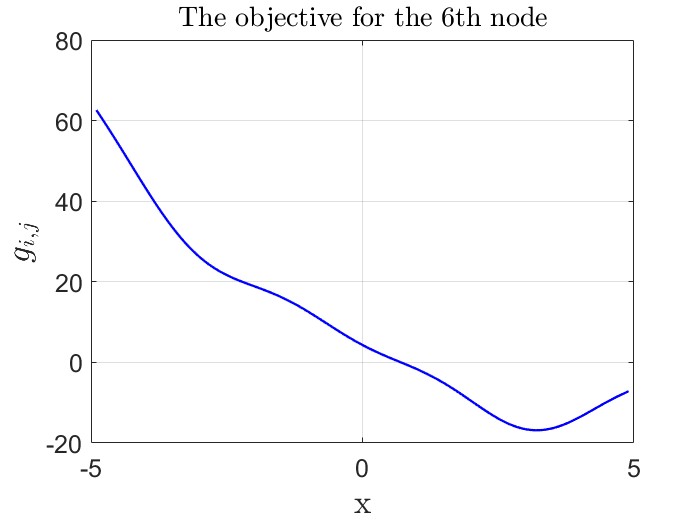}
		\includegraphics[width=1.75in]{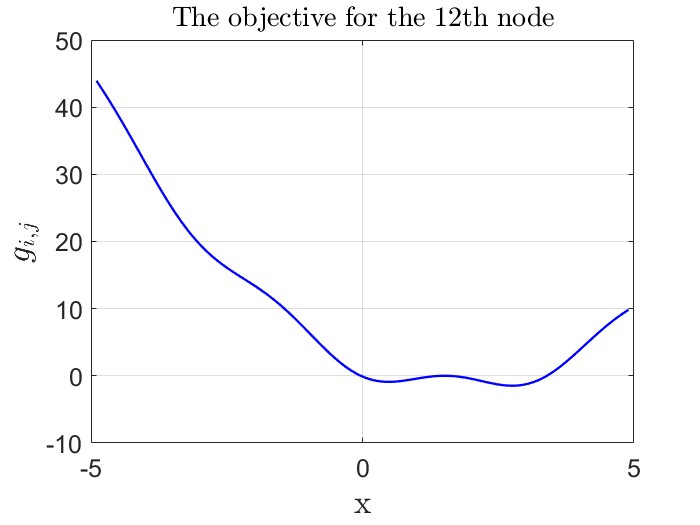} 
		\includegraphics[width=1.75in]{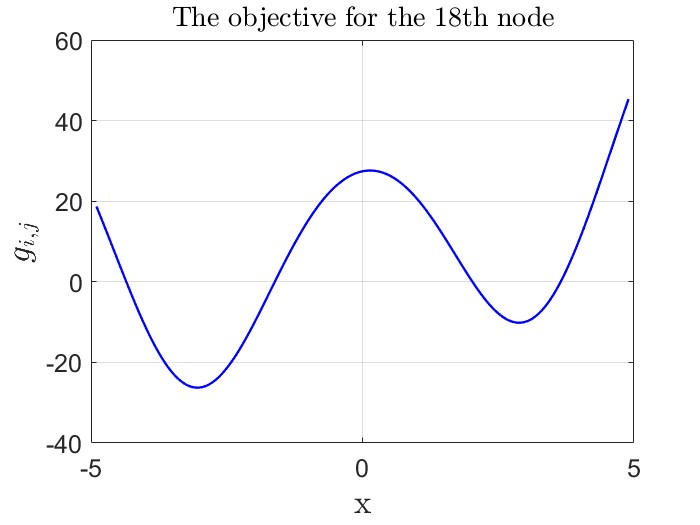}   
		\caption{The top-left figure shows the global objective function $ \frac{1}{n}\sum_{i=1}^{n} \frac{1}{m}\sum_{j=1}^{N} g_{i,j}(y_i)$. The other four figures show the local non-convex objective functions at four sample nodes. This shows an example non-convex local objective function $g_i(\cdot)$ satisfying Assumption~\ref{ass_cost}.}  \label{fig_nonconv}
	\end{figure}
	We set the gradient-tracking parameter as $\alpha = 0.1$ and the log-scale parameter as $\rho = 0.03125,0.125$. The CPU scheduling parameters are $\epsilon=2$,  $\sigma=2$. The simulation result is shown in Fig.~\ref{fig_nonconv2}.
	\begin{figure}
		\centering
		\includegraphics[width=2.5in]{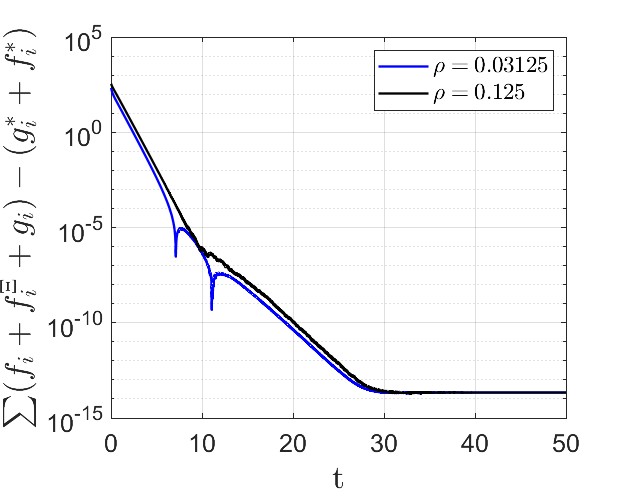}  
		\includegraphics[width=2.5in]{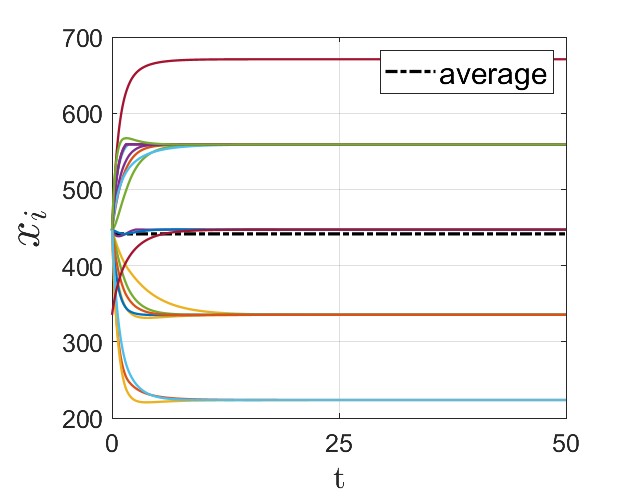}  
		\caption{These figures show the time-evolution of the  locally non-convex cost-function~\eqref{eq_fij_sim} plus CPU costs (left) and assigned CPU resources (right)  subject to log-scale quantization. }  \label{fig_nonconv2}
	\end{figure}

\subsection{Real Data-Set Example} \label{sec_sim2}
Next, we consider the MNIST image data set and compare the performance of our ML setup with some existing literature for image classification.
We use $N = 12000$ labelled images from this data set and classify them using logistic regression with a convex regularizer over a network of $n=16$ nodes.  
	The objective optimization model is defined as 
	\begin{align}
		\min_{\mb{b},c} &
		F(\mb{b},c) = \frac{1}{n}\sum_{i=1}^{n} f_i,
	\end{align}  
	with every node $i$ accessing a batch of $m_i=\frac{N}{n}=750$ sample images. Then, every node locally minimizes the following training loss function:
	\begin{align}\label{eq_fij_regression}
		f_i(\mb{x}) = \frac{1}{m_i}\sum_{j=1}^{m_i} \ln(1+\exp(-(\mb{b}^\top \mc{X}_{i,j}+c)\mc{Y}_{i,j}))+\frac{\theta}{2}\|\mb{b}\|_2^2,
	\end{align}
	which is smooth because of the addition of the regularizer $\theta$. The $j$-th sample at node $i$ is defined as a tuple ${\mc{X}_{i,j},\mc{Y}_{i,j}} \subseteq \mathbb{R}^{784}
	\times \{+1, -1\}$ and $\mb{b},c$ are the regression objective parameters to be optimized. This training objective model is convex and satisfies Assumption~\ref{ass_cost} for finite number of bounded-value data points. The network of computing nodes is considered as an \textit{exponential} graph, where, following Remark~\ref{rem_data}, this type of network structure gives a lower optimality gap and faster convergence. Some existing (non-quantized) algorithms are used for comparison: \textit{GP} \cite{nedic2014distributed}, \textit{SGP} \cite{spiridonoff2020robust}, \textit{S-ADDOPT} \cite{qureshi2020s}, \textit{ADDOPT} \cite{xi2017add}, and \textit{PushSAGA} \cite{qureshi2021push}. For all algorithms, the step sizes are hand-tuned to achieve the best performance.
Fig.~\ref{fig_compare} shows the comparison results on the MNIST data set. Clearly, the proposed log-quantized algorithm (with $\alpha = 4.1$ and $\rho = 0.0625$) shows good performance as compared to these works.

\begin{figure}[]
	\centering
	\includegraphics[width=2.75in]{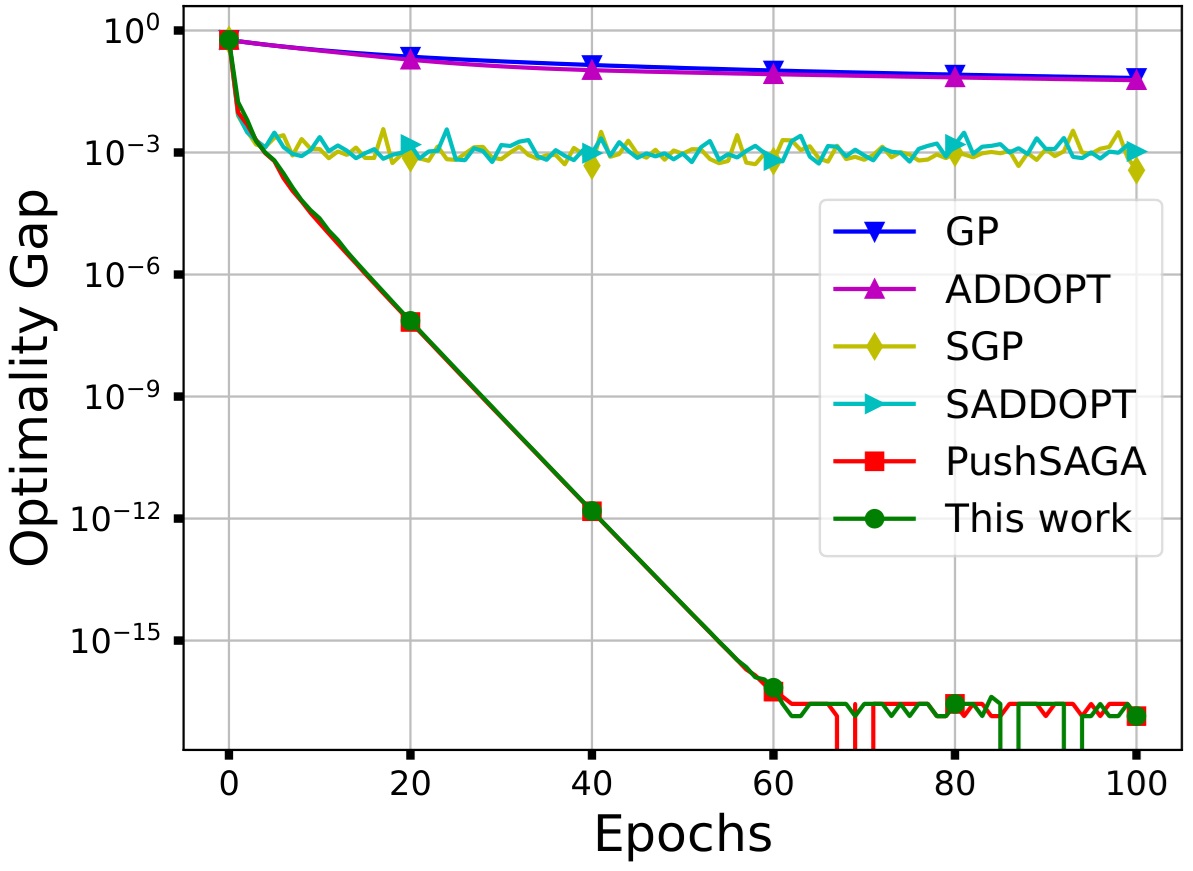} 
	\caption{The performance comparison of the proposed ML optimization algorithm on the MNIST dataset with some existing literature in terms of optimality gap (cost residual) and rate of convergence.}
	\label{fig_compare}
\end{figure}

\section{Conclusions}\label{sec_con}
\subsection{Concluding Remarks}
This work provides a distributed co-optimization algorithm for CPU scheduling and different ML applications over dynamic networks. {One advantage is that the algorithm optimally assigns the computing resources with all-time resource-demand feasibility and can be terminated at any time depending on the termination criteria. Moreover, the proof is based on eigenspectrum perturbation analysis and the Lyapunov stability theorem, which allows convergence over time-varying networks. On the other hand, although Assumption~\ref{ass_net} assumes undirected networks, the optimization of the ML objective alone may converge over weight-balanced directed networks. A limitation is solving the general co-optimization objective \eqref{eq_opt} over directed networks, which is a direction of our current research.   
Both simulation results and mathematical analysis show the optimal convergence and all-time feasibility of the algorithm.

\subsection{Future Directions}
One direction of our future research is to relax Assumption 1 to more general non-convex objective models such that the algorithm can be applied to more complex models, e.g., deep neural networks.
As other future research directions, the results can be used for CPU scheduling and computing resource management in different distributed optimization, learning, and filtering applications. For example, distributed optimization and filtering over large transportation networks, computing resource allocation for detection and estimation over large-scale CPS and IoT environments, and resource scheduling for multi-target tracking over large-scale sensor networks. Considering node failure or loss of communications by adding redundancy mechanisms in the network of computing nodes is another interesting research direction.

%\section*{Acknowledgement}
%Authors would like to thank Usman A. Khan and Muhammad~I.~Qureshi for sharing their codes for MNIST data set classification and for their insightful comments. This work has been supported by the Center for International Scientific Studies \& Collaborations (CISSC), Ministry of Science Research and Technology of Iran.

\bibliographystyle{elsarticle-num}
\bibliography{bibliography}

\begin{thebibliography}{10}
\expandafter\ifx\csname url\endcsname\relax
  \def\url#1{\texttt{#1}}\fi
\expandafter\ifx\csname urlprefix\endcsname\relax\def\urlprefix{URL }\fi
\expandafter\ifx\csname href\endcsname\relax
  \def\href#1#2{#2} \def\path#1{#1}\fi

\bibitem{ARANDA2022107826}
J.~A.~S. Aranda, R.~{dos Santos Costa}, V.~W. {de Vargas}, P.~R. {da Silva
  Pereira}, J.~L.~V. Barbosa, M.~P. Vianna, Context-aware edge computing and
  internet of things in smart grids: A systematic mapping study, Computers and
  Electrical Engineering 99 (2022) 107826.

\bibitem{zamani2016iterative}
H.~Zamani, H.~Zayyani, F.~Marvasti, An iterative dictionary learning-based
  algorithm for {DOA} estimation, IEEE Communications Letters 20~(9) (2016)
  1784--1787.

\bibitem{WANG2025109825}
J.~Wang, Multi agent system based smart grid anomaly detection using blockchain
  machine learning model in mobile edge computing network, Computers and
  Electrical Engineering 121 (2025) 109825.

\bibitem{yosuf2020energy}
B.~A. Yosuf, M.~Musa, T.~Elgorashi, J.~Elmirghani, Energy efficient distributed
  processing for {IoT}, IEEE Access 8 (2020) 161080--161108.

\bibitem{CAPOTA2019204}
E.~A. Capota, C.~S. Stangaciu, M.~V. Micea, D.~Curiac, Towards mixed
  criticality task scheduling in cyber physical systems: Challenges and
  perspectives, Journal of Systems and Software 156 (2019) 204--216.

\bibitem{MELONI2018156}
A.~Meloni, P.~A. Pegoraro, L.~Atzori, A.~Benigni, S.~Sulis, Cloud-based iot
  solution for state estimation in smart grids: Exploiting virtualization and
  edge-intelligence technologies, Computer Networks 130 (2018) 156--165.

\bibitem{di2020distributed}
P.~Di~Lorenzo, S.~Barbarossa, S.~Sardellitti, Distributed signal processing and
  optimization based on in-network subspace projections, IEEE Transactions on
  Signal Processing 68 (2020) 2061--2076.

\bibitem{spl24}
M.~Doostmohammadian, A.~Aghasi, Accelerated distributed allocation, IEEE Signal
  Processing Letters 31 (2024) 651--655.

\bibitem{liu2021novel}
A.~Liu, P.~B. Luh, B.~Yan, M.~A. Bragin, A novel integer linear programming
  formulation for job-shop scheduling problems, IEEE Robotics and Automation
  Letters 6~(3) (2021) 5937--5944.

\bibitem{bragin2018scalable}
M.~A. Bragin, P.~B. Luh, B.~Yan, X.~Sun, A scalable solution methodology for
  mixed-integer linear programming problems arising in automation, IEEE
  Transactions on Automation Science and Engineering 16~(2) (2018) 531--541.

\bibitem{bragin2020distributed}
M.~A. Bragin, B.~Yan, P.~B. Luh, Distributed and asynchronous coordination of a
  mixed-integer linear system via surrogate lagrangian relaxation, IEEE
  Transactions on Automation Science and Engineering 18~(3) (2020) 1191--1205.

\bibitem{MiadCons}
H.~Sayyaadi, M.~Moarref, A distributed algorithm for proportional task
  allocation in networks of mobile agents, IEEE Transactions on Automatic
  Control 56~(2) (2011) 405--410.
\newblock \href {http://dx.doi.org/10.1109/TAC.2010.2089653}
  {\path{doi:10.1109/TAC.2010.2089653}}.

\bibitem{MSC09}
M.~Doostmohammadian, H.~Sayyaadi, M.~Moarref, A novel consensus protocol using
  facility location algorithms, in: IEEE Conf. on Control Applications \&
  Intelligent Control, 2009, pp. 914--919.

\bibitem{grammenos2023cpu}
A.~Grammenos, T.~Charalambous, E.~Kalyvianaki, {CPU} scheduling in data centers
  using asynchronous finite-time distributed coordination mechanisms, IEEE
  Transactions on Network Science and Engineering 10~(4) (2023) 1880--1894.

\bibitem{kalyvianaki2009self}
E.~Kalyvianaki, T.~Charalambous, S.~Hand, Self-adaptive and self-configured cpu
  resource provisioning for virtualized servers using kalman filters, in:
  Proceedings of the 6th international conference on Autonomic computing, 2009,
  pp. 117--126.

\bibitem{rikos2021optimal}
A.~I. Rikos, A.~Grammenos, E.~Kalyvianaki, C.~N. Hadjicostis, T.~Charalambous,
  K.~H. Johansson, Optimal {CPU} scheduling in data centers via a finite-time
  distributed quantized coordination mechanism, in: 60th IEEE Conference on
  Decision and Control (CDC), IEEE, 2021, pp. 6276--6281.

\bibitem{koutsoukis2016online}
N.~C. Koutsoukis, D.~O. Siagkas, P.~S. Georgilakis, N.~D. Hatziargyriou, Online
  reconfiguration of active distribution networks for maximum integration of
  distributed generation, IEEE Transactions on Automation Science and
  Engineering 14~(2) (2016) 437--448.

\bibitem{ZHOLBARYSSOV201947}
M.~Zholbaryssov, D.~Fooladivanda, A.~D. Dominguez-Garcia, Resilient distributed
  optimal generation dispatch for lossy {AC} microgrids, Systems \& Control
  Letters 123 (2019) 47--54.

\bibitem{scl2023}
M.~Doostmohammadian, Distributed energy resource management: All-time
  resource-demand feasibility, delay-tolerance, nonlinearity, and beyond, IEEE
  Control Systems Letters 7 (2023) 3423 -- 3428.

\bibitem{cheng2024distributed}
L.~Cheng, S.~Zhang, Y.~Wang, Distributed optimal capacity allocation of
  integrated energy system via modified {ADMM}, Applied Mathematics and
  Computation 465 (2024) 128369.

\bibitem{scl}
M.~Doostmohammadian, A.~Aghasi, M.~Vrakopoulou, H.~R. Rabiee, U.~A. Khan,
  T.~Charalambous, Distributed delay-tolerant strategies for
  equality-constraint sum-preserving resource allocation, Systems \& Control
  Letters 182 (2023) 105657.

\bibitem{dsp}
M.~Doostmohammadian, H.~R. Rabiee, Distributed automatic generation control
  subject to ramp-rate-limits: Anytime feasibility and uniform
  network-connectivity, Digital Signal Processing (2025) 105576.

\bibitem{doan2020fast}
T.~T. Doan, S.~T. Maguluri, J.~Romberg, Fast convergence rates of distributed
  subgradient methods with adaptive quantization, IEEE Transactions on
  Automatic Control 66~(5) (2020) 2191--2205.

\bibitem{rikos2022non}
A.~I. Rikos, C.~N. Hadjicostis, K.~H. Johansson, Non-oscillating quantized
  average consensus over dynamic directed topologies, Automatica 146 (2022)
  110621.

\bibitem{wu2018error}
J.~Wu, W.~Huang, J.~Huang, T.~Zhang, Error compensated quantized sgd and its
  applications to large-scale distributed optimization, in: International
  Conference on Machine Learning, PMLR, 2018, pp. 5325--5333.

\bibitem{kajiyama2020linear}
Y.~Kajiyama, N.~Hayashi, S.~Takai, Linear convergence of consensus-based
  quantized optimization for smooth and strongly convex cost functions, IEEE
  Transactions on Automatic Control 66~(3) (2020) 1254--1261.

\bibitem{jian2019distributed}
L.~Jian, J.~Hu, J.~Wang, K.~Shi, Distributed inexact dual consensus {ADMM} for
  network resource allocation, Optimal Control Applications and Methods 40~(6)
  (2019) 1071--1087.

\bibitem{falsone2023augmented}
A.~Falsone, M.~Prandini, Augmented lagrangian tracking for distributed
  optimization with equality and inequality coupling constraints, Automatica
  157 (2023) 111269.

\bibitem{gong2024primal}
K.~Gong, L.~Zhang, Primal-dual algorithm for distributed optimization with
  coupled constraints, Journal of Optimization Theory and Applications 201~(1)
  (2024) 252--279.

\bibitem{ddsvm}
M.~Doostmohammadian, W.~Jiang, M.~Liaquat, A.~Aghasi, H.~Zarrabi, Discretized
  distributed optimization over dynamic digraphs, IEEE Transactions on
  Automation Science and Engineering 22 (2024) 2758--2767.

\bibitem{xin2020decentralized}
R.~Xin, S.~Kar, U.~A. Khan, Decentralized stochastic optimization and machine
  learning: A unified variance-reduction framework for robust performance and
  fast convergence, IEEE Signal Processing Magazine 37~(3) (2020) 102--113.

\bibitem{sundhar2012new}
S.~Sundhar~Ram, A.~Nedi{\'c}, V.~V. Veeravalli, A new class of distributed
  optimization algorithms: Application to regression of distributed data,
  Optimization Methods and Software 27~(1) (2012) 71--88.

\bibitem{dimakis2010gossip}
A.~G. Dimakis, S.~Kar, J.~M.~F. Moura, M.~G. Rabbat, A.~Scaglione, Gossip
  algorithms for distributed signal processing, Proceedings of the IEEE 98~(11)
  (2010) 1847--1864.

\bibitem{kar2008distributed}
S.~Kar, J.~M.~F. Moura, Distributed consensus algorithms in sensor networks
  with imperfect communication: Link failures and channel noise, IEEE
  Transactions on Signal Processing 57~(1) (2008) 355--369.

\bibitem{zhang2023top}
X.~Zhang, M.~M. Vasconcelos, Top-k data selection via distributed sample
  quantile inference, in: Learning for Dynamics and Control Conference, PMLR,
  2023, pp. 813--824.

\bibitem{nesterov1998introductory}
Y.~Nesterov, Introductory lectures on convex programming, {volume I}: Basic
  course, Lecture notes 3~(4) (1998) 5.

\bibitem{bertsekas1975necessary}
D.~P. Bertsekas, Necessary and sufficient conditions for a penalty method to be
  exact, Mathematical programming 9~(1) (1975) 87--99.

\bibitem{chapelle2007training}
O.~Chapelle, Training a support vector machine in the primal, Neural
  computation 19~(5) (2007) 1155--1178.

\bibitem{dogan2016unified}
U.~Dogan, T.~Glasmachers, C.~Igel, A unified view on multi-class support vector
  classification, The Journal of Machine Learning Research 17~(1) (2016)
  1550--1831.

\bibitem{slp_book}
D.~Jurafsky, J.~H. Martin, Speech and Language Processing, Prentice Hall, 2020.

\bibitem{zhang2003modified}
J.~Zhang, R.~Jin, Y.~Yang, A.~G. Hauptmann, Modified logistic regression: An
  approximation to {SVM} and its application in large-scale text
  categorization, in: Proceeding 20th International Conference on Machine
  Learinig (ICML), Washington, DC, 2003.

\bibitem{qureshi2021decentralized}
M.~I. Qureshi, R.~Xin, S.~Kar, U.~A. Khan, A decentralized variance-reduced
  method for stochastic optimization over directed graphs, in: IEEE
  International Conference on Acoustics, Speech and Signal Processing (ICASSP),
  IEEE, 2021, pp. 5030--5034.

\bibitem{mcmahan2017communication}
B.~McMahan, E.~Moore, D.~Ramage, S.~Hampson, B.~A. Arcas,
  Communication-efficient learning of deep networks from decentralized data,
  in: Artificial intelligence and statistics, PMLR, 2017, pp. 1273--1282.

\bibitem{SensNets:Olfati04}
R.~Olfati-Saber, R.~M. Murray, Consensus problems in networks of agents with
  switching topology and time-delays, IEEE Transactions on Automatic Control
  49, no. 9 (2004) 1520--1533.

\bibitem{olfatisaberfaxmurray07}
R.~Olfati-Saber, J.~A. Fax, R.~M. Murray, Consensus and cooperation in
  networked multi-agent systems, IEEE Proceedings 95~(1) (2007) 215--233.

\bibitem{Boyd-CVXBook}
S.~Boyd, L.~Vandenberghe, Convex Optimization, Cambridge University Press, New
  York, NY, USA, 2004.

\bibitem{bertsekas_lecture}
D.~P. Bertsekas, A.~Nedic, A.~E. Ozdaglar, Convexity, duality, and {L}agrange
  multipliers, Lecture Notes, MIT Press.

\bibitem{stewart_book}
G.~W. Stewart, J.~Sun, Matrix perturbation theory, Academic Press, 1990.

\bibitem{cai2012average}
K.~Cai, H.~Ishii, Average consensus on general strongly connected digraphs,
  Automatica 48~(11) (2012) 2750--2761.

\bibitem{delay_est}
M.~Doostmohammadian, M.~Pirani, U.~A. Khan, T.~Charalambous, Consensus-based
  distributed estimation in the presence of heterogeneous, time-invariant
  delays, IEEE Control Systems Letters 6 (2021) 1598 -- 1603.

\bibitem{nonlin}
J.~E. Slotine, W.~Li, Applied nonlinear control, Prentice hall Englewood
  Cliffs, NJ, 1991.

\bibitem{zhu2016quantized}
S.~Zhu, M.~Hong, B.~Chen, Quantized consensus {ADMM} for multi-agent
  distributed optimization, in: IEEE International Conference on Acoustics,
  Speech and Signal Processing (ICASSP), IEEE, 2016, pp. 4134--4138.

\bibitem{danaee2021energy}
A.~Danaee, R.~de~Lamare, V.~Nascimento, Energy-efficient distributed learning
  with coarsely quantized signals, IEEE Signal Processing Letters 28 (2021)
  329--333.

\bibitem{hanna2021quantization}
O.~A. Hanna, Y.~H. Ezzeldin, C.~Fragouli, S.~Diggavi, Quantization of
  distributed data for learning, IEEE Journal on Selected Areas in Information
  Theory 2~(3) (2021) 987--1001.

\bibitem{bastianello2023online}
N.~Bastianello, A.~I. Rikos, K.~H. Johansson, Online distributed learning with
  quantized finite-time coordination, in: 62nd IEEE Conference on Decision and
  Control (CDC), IEEE, 2023, pp. 5026--5032.

\bibitem{xin2021fast}
R.~Xin, U.~A. Khan, S.~Kar, A fast randomized incremental gradient method for
  decentralized nonconvex optimization, IEEE Transactions on Automatic Control
  67~(10) (2021) 5150--5165.

\bibitem{nedic2014distributed}
A.~Nedi{\'c}, A.~Olshevsky, Distributed optimization over time-varying directed
  graphs, IEEE Transactions on Automatic Control 60~(3) (2014) 601--615.

\bibitem{spiridonoff2020robust}
A.~Spiridonoff, A.~Olshevsky, I.~Paschalidis, Robust asynchronous stochastic
  gradient-push: Asymptotically optimal and network-independent performance for
  strongly convex functions, Journal of machine learning research 21~(58).

\bibitem{qureshi2020s}
M.~I. Qureshi, R.~Xin, S.~Kar, U.~A. Khan, {S-ADDOPT}: Decentralized stochastic
  first-order optimization over directed graphs, IEEE Control Systems Letters
  5~(3) (2020) 953--958.

\bibitem{xi2017add}
C.~Xi, R.~Xin, U.~A. Khan, {ADD-OPT:} accelerated distributed directed
  optimization, IEEE Transactions on Automatic Control 63~(5) (2017)
  1329--1339.

\bibitem{qureshi2021push}
M.~I. Qureshi, R.~Xin, S.~Kar, U.~A. Khan, {Push-SAGA:} a decentralized
  stochastic algorithm with variance reduction over directed graphs, IEEE
  Control Systems Letters 6 (2021) 1202--1207.

\end{thebibliography}

\end{document}